\documentclass{article} 
\usepackage{fullpage}
\usepackage{times}
\usepackage{hyperref}
\usepackage{url}
\usepackage{amsthm}
\usepackage{amsmath}
\usepackage{amsfonts}
\usepackage{amssymb}
\usepackage{etoolbox}
\usepackage{xspace}

\newcommand{\bE}{\mathbb{E}}
\newcommand{\bP}{\mathbb{P}}
\newcommand{\bR}{\mathbb{R}}

\newcommand{\cN}{\mathcal{N}}

\newcommand{\KL}{\mathop{\mathrm{KL}}}
\newcommand{\TV}{\mathop{\mathrm{TV}}}

\newcommand{\argmin}{\mathop{\mathrm{argmin}}}
\newcommand{\argmax}{\mathop{\mathrm{argmax}}}
\newcommand{\simiid}{\mathop{\mathrm{\sim}}\limits^{\text{i.i.d.}}}

\newtheorem{theorem}{Theorem}

\newtheorem{lemma}{Lemma}
\newtheorem{prop}{Proposition}

\newtoggle{detailed}
\toggletrue{detailed}
\newtoggle{detailed2}
\toggletrue{detailed2}
\newtoggle{detailed3}
\toggletrue{detailed3}

\usepackage[round]{natbib}

\let\hat\widehat

\title{Minimax Theory for High-dimensional Gaussian Mixtures with
Sparse Mean Separation\footnote{This work is supported in part by NSF grant IIS-1116458 and NSF CAREER award IIS-1252412.}}

\date{}
\author{Martin Azizyan \qquad Aarti Singh \qquad Larry Wasserman\\ Carnegie-Mellon University}

\begin{document}

\maketitle

\begin{abstract}
While several papers have investigated computationally and statistically 
efficient methods for learning Gaussian mixtures, 
precise minimax bounds
for their statistical performance as well as fundamental limits in high-dimensional 
settings are not well-understood. In this paper, we provide precise information theoretic 
bounds on the clustering accuracy and sample complexity of learning a mixture of two isotropic Gaussians in high
dimensions under small mean separation. 
If there 
is a sparse subset of relevant dimensions that determine the mean separation, 
then the sample complexity only depends on the number of relevant dimensions and mean
separation, and can be achieved by a simple computationally efficient procedure. 
Our results provide the first step of a theoretical basis for recent
methods that combine feature selection and clustering.
\end{abstract}
\vspace{-0.1in}

\section{Introduction}
\vspace{-0.05in}
Gaussian mixture models provide a simple 
framework for
several 
machine learning problems including clustering, density estimation
and classification.
Mixtures are especially appealing in high dimensional problems.
Perhaps the most common use of Gaussian mixtures
is for clustering.
Of course,
the statistical (and computational) behavior
of these methods can degrade
in high dimensions.
Inspired by the success of variable selection methods in regression,
several authors have considered
variable selection
for clustering.
However, there appears to be no theoretical
results justifying the advantage of variable selection in high dimensional
setting. 

To see why some sort of variable selection
might be useful, consider
clustering $n$ subjects using a vector of
$d$ genes for each subject.
Typically $d$ is much larger than $n$
which suggests that statistical clustering methods
will perform poorly.
However, it may be the case that there are only
a small number of relevant genes
in which case we might expect better behavior
by focusing on this small set of relevant genes.

The purpose of this paper
is to provide precise bounds on clustering error
with mixtures of Gaussians.
We consider both the general case where all features
are relevant, and the special case where
only a subset of features are relevant.
Mathematically, we model an irrelevant feature by
requiring the mean of that feature to be the same
across clusters, so that the feature does not serve
 to differentiate the groups. 
Throughout this paper,
we use the probability of misclustering
an observation, relative to the optimal clustering
if we had known the true distribution,
as our loss function.
This is akin to using excess risk in classification.

This paper makes the following contributions:
\vspace{-0.05in}
\begin{itemize}
\item We provide information theoretic bounds on the sample complexity 
of learning a mixture of
two isotropic Gaussians in the small mean separation setting that precisely captures
the dimension dependence, and matches known sample complexity requirements for 
some existing algorithms. This also debunks the 
myth that there is a gap between statistical and computational complexity of learning 
mixture of two isotropic Gaussians for small mean separation. 
Our bounds require non-standard arguments since our loss function does
not satisfy the triangle inequality.
\item We consider the high-dimensional setting where only a subset of relevant dimensions 
determine the mean separation between mixture components and demonstrate that learning 
is substantially easier as the sample complexity only depends on the sparse set of relevant dimensions.
This provides some theoretical basis for feature selection approaches to clustering.
\item We show that a simple computationally feasible procedure nearly achieves the information theoretic sample complexity even in high-dimensional sparse mean separation settings. 
\end{itemize}

{\bf Related Work.}
There is a long and continuing history
of research on mixtures of Gaussians.
A complete review is not feasible but we mention
some highlights of the work most related to ours.

Perhaps the most popular method for estimating 
a mixture distribution is
maximum likelihood. 
Unfortunately, maximizing the likelihood is NP-Hard. 
This has led to a 
stream of work on alternative methods
for estimating mixtures.
These new algorithms use
pairwise distances, spectral methods or the method of moments.

Pairwise methods are
developed in
\cite{dasgupta1999learning}, \cite{schulman00} and
\cite{sanjeev2001learning}.
These methods require the mean separation to increase
with dimension. The first one requires the separation to be 
$\sqrt{d}$ while the latter two improve it to $d^{1/4}$.
To avoid this problem,
\cite{vempala2004spectral}
introduced the idea of using spectral methods
for estimating mixtures of spherical Gaussians which makes 
mean separation independent of dimension.
The assumption that the components
are spherical
was removed in
\cite{brubaker2008isotropic}. 
Their method only requires the 
components to be separated by a hyperplane 
and runs in polynomial time,  
but requires 
$n = \Omega(d^4 \log d)$
samples.
Other spectral methods include
\cite{kannan2005spectral}, \cite{achlioptas2005spectral} and
\cite{hsu2013learning}.
The latter uses clever spectral decompositions together
with the method of moments to derive
an effective algorithm.

\cite{kalai2012disentangling}
use the method of moments to get estimates
without requiring separation between components of the mixture components.
A similar approach is given in
\cite{belkin2010polynomial}.
\cite{chaudhuri2009learning}
give a modified $k$-means algorithm for
estimating a mixture of two Gaussians.
For the large
mean separation setting $\mu > 1$, \cite{chaudhuri2009learning} 
show that 
$n = \tilde \Omega(d/\mu^2)$ samples
are needed. They also provide an information theoretic 
bound on the necessary  
sample complexity of any algorithm which matches the sample complexity 
of their method 
(up to log factors) in $d$ and $\mu$. 
When the mean separation is small $\mu < 1$, 
they show that
$n = \tilde\Omega(d/\mu^4)$ samples
are sufficient for accurate estimation. Our results for the 
small mean separation setting provide a 
matching necessary condition.

Most of these papers
are concerned with computational efficiency
and do not give precise,
statistical minimax upper and lower bounds.
None of them deal with the case we are interested in,
namely, a high dimensional mixture with sparse
mean separation.

We should also point out that the results in
different papers are not necessarily
comparable since
different authors use different loss functions.
In this paper we use the probability
of misclassifying a future observation,
relative to how the correct distribution
clusters the observation, as our loss function.
This should not be confused with
the probability of attributing a new observation
to the wrong component of the mixture.
The latter loss does not to tend to zero
as the sample size increases.
Our loss is similar to the excess risk
used in classification where
we compare the misclassification rate of a classifier 
to the misclassification rate of the Bayes optimal classifier.

Finally,
we remind the reader that
our motivation for studying
sparsely separated mixtures
is that this provides a model for
variable selection in clustering problems.
There are some relevant recent
papers on this problem in the high-dimensional setting.
\citet{pan2007penalized}
use penalized mixture models
to do variable selection
and clustering simultaneously.
\cite{witten2010framework} develop
a penalized version of $k$-means clustering.
Related methods include
\citet{raftery2006variable,sun2012regularized}
and \citet{guo2010pairwise}.
The applied bioinformatics literature also
contains a huge number of heuristic methods
for this problem.
None of these papers provide
minimax bounds
for the clustering error or provide theoretical 
evidence of the benefit of using variable selection
in unsupervised problems such as clustering.

\section{Problem Setup}

In this paper, we consider the simple setting of learning a mixture of two isotropic 
Gaussians with equal mixing weights, given $n$ data points $X_1, \dots, X_n \in 
\bR^d$ drawn i.i.d. from a
$d$-dimensional mixture density function 
$$p_\theta(x)=\frac{1}{2} f(x;\mu_1,\sigma^2I)+\frac{1}{2} f(x;\mu_2,\sigma^2I),$$ 
where $f(\cdot;\mu,\Sigma)$ is the density of $\cN(\mu,\Sigma)$, $\sigma>0$ is a fixed
constant, and $\theta: = (\mu_1, \mu_2)\in\Theta$. 
We consider two classes $\Theta$ of parameters:
$$
\Theta_{\lambda}=\left\{(\mu_1,\mu_2): \|\mu_1-\mu_2\|\geq\lambda\right\}
$$
$$
\Theta_{\lambda,s}=
\left\{(\mu_1,\mu_2): \|\mu_1-\mu_2\|\geq\lambda,\;\|\mu_1-\mu_2\|_0\leq s\right\}\subseteq \Theta_{\lambda}.
$$
The first class defines mixtures where the components have a mean separation of at least 
$\lambda >0$. The second class defines mixtures with mean separation $\lambda > 0$ 
along a sparse set of $s \in \{1,\dots, d\}$ dimensions. 
Also, let $P_\theta$ denote the probability measure corresponding to $p_\theta$.
Throughout the paper, we will use $\phi$ and $\Phi$ to denote the standard normal 
density and distribution functions.

For a mixture with parameter $\theta$, the Bayes optimal classification, that is, assignment 
of a point $x \in \bR^d$ to the correct mixture component,
is given by the function 
$$
F_\theta(x)=\argmax\limits_{i\in\{1,2\}} f(x;\mu_i,\sigma^2I).
$$
Given any other candidate assignment function 
$F:\bR^d\rightarrow\{1,2\}$, we define the loss incurred by $F$ as
$$
L_\theta(F)=\min\limits_\pi P_\theta(\{x:F_\theta(x)\neq \pi(F(x))\})
$$ where 
the minimum is over all permutations $\pi:\{1,2\}\rightarrow\{1,2\}$. 
This is the probability of misclustering relative to
an oracle that uses the true distribution
to do optimal clustering.

We denote by $\widehat F_n$ any assignment function learned from
the data
$X_1, \dots, X_n$, also referred to as estimator.  The goal of this
paper is to quantify how the minimax expected loss (worst case
expected loss for the best estimator) 
$$
R_n \equiv \inf_{\widehat   F_n}\sup_{\theta\in\Theta} \bE_\theta L_\theta(\widehat F_n)
$$ 
scales
with number of samples $n$, the dimension of the feature space $d$, 
the number of
relevant dimensions $s$, and the signal-to-noise ratio defined as the
ratio of mean separation to standard deviation $\lambda/\sigma$.  We
will also demonstrate a specific estimator that achieves the minimax
scaling.


For the purposes of this paper,
we say that feature $j$ is irrelevant if
$\mu_1(j) = \mu_2(j)$.
Otherwise we say that feature $j$ is relevant.

\section{Minimax Bounds}

\subsection{Small mean separation setting without sparsity}

We begin without assuming any sparsity, that is,
all features are relevant. In this case, comparing the projections
of the data to the projection of the sample mean onto the first 
principal component suffices to achieve both minimax optimal sample
complexity and clustering loss. 

\begin{theorem}[Upper bound]
\label{thm:u5-ns-upper}
Define
$$
\widehat{F}_n(x) = \left\{
\begin{array}{rl}
1 & \mbox{if } x^Tv_1(\widehat{\Sigma}_n)\geq \widehat{\mu}_n^Tv_1(\widehat{\Sigma}_n) \\
2 & \mbox{otherwise.} 
\end{array}
\right.
$$
where $\widehat \mu_n = n^{-1}\sum^n_{i=1} X_i$ is the sample mean, 
$\widehat \Sigma_n = n^{-1} \sum^n_{i=1}(X_i - \widehat \mu_n)(X_i - \widehat \mu_n)^T$
is the sample covariance and $v_1(\widehat \Sigma_n)$ denotes the eigenvector 
corresponding to the largest eigenvalue of $\widehat \Sigma_n$.
If $n\geq \max(68,4d)$, then
\begin{align*}
\sup_{\theta\in\Theta_\lambda} \bE_\theta L_\theta(\widehat{F})
\leq 600 \max\left(\frac{4\sigma^2}{\lambda^2}, 1\right)   \sqrt{\frac{d\log(nd)}{n}}.
\end{align*}
Furthermore, if $\frac{\lambda}{\sigma} \geq 2\max(80,14\sqrt{5d})$, then
\begin{align*}
\sup_{\theta\in\Theta_\lambda} \bE_\theta L_\theta(\widehat{F})
\leq 17\exp\left(-\frac{n}{32}\right)
+ 9\exp\left(-\frac{\lambda^2}{80\sigma^2} \right).
\end{align*}
\end{theorem}

\begin{theorem}[Lower bound]
\label{thm:lb-nonsparse}
Assume that $d\geq 9$ and $\frac{\lambda}{\sigma}\leq 0.2$.
Then
\begin{align*}
\inf_{\widehat{F}_n}\sup_{\theta\in\Theta_\lambda} \bE_\theta L_\theta(\widehat{F}_n) 
&\geq \frac{1}{500}   
\min\left\{\frac{\sqrt{\log 2}}{3}\frac{\sigma^2}{\lambda^2}\sqrt{\frac{d-1}{n}}, \frac{1}{4}\right\}.
\end{align*}
\end{theorem}

We believe that some of the constants (including lower 
bound on $d$ and exact upper bound on $\lambda/\sigma$) 
can be tightened, but the results demonstrate matching scaling behavior
of clustering error with $d$, $n$ and $\lambda/\sigma$.
Thus, we see (ignoring constants and log terms) that
$$
R_n \approx 
\frac{\sigma^2}{\lambda^2} \sqrt{\frac{d}{n}},
\quad
\text{or equivalently}
\quad
n \approx \frac{d}{\lambda^4/\sigma^4} \ \text{for a constant target value of} \ R_n.
$$
The result is quite intuitive:
the dependence on dimension $d$ is as expected.
Also we see that the rate depends
in a precise way on the signal-to-noise ratio
$\lambda/\sigma$. 
In particular,
the results imply that we need $d\leq n$. 

In modern high-dimensional datasets, we often have $d> n$ i.e. 
large number of features and not enough samples. However, inference
is usually tractable since not all features are relevant to the learning task
at hand. This sparsity of relevant feature set has been successfully exploited
in supervised learning problems such as regression and classification. 
We show next that the same is true for clustering under the Gaussian mixture model.

\subsection{Sparse and small mean separation setting}

Now we consider the case where there are $s < d$ relevant features.
Let $S$ denote the set of relevant features.
We begin by constructing an estimator $\hat S_n$ of $S$ as follows.
Define
$$
\widehat{\tau}_n = \frac{1+\alpha}{1-\alpha} \min_{i\in \{1,\dots, d\}} \widehat{\Sigma}_n(i,i),
$$
where 
$$
\alpha = \sqrt{\frac{6\log (nd) }{n}}  + \frac{2\log(nd)}{n}.
$$
Now let
$$
\widehat{S}_n=\{i\in \{1,\dots, d\}: \widehat{\Sigma}_n(i,i) > \widehat{\tau}_n \}.
$$

Now we use the same method as before, but using only the features 
in $\widehat S_n$ identified as relevant.

\begin{theorem}[Upper bound]
\label{thm:upperbound-new-support-recovery-plus-estimation}
Define
$$
\widehat{F}_n(x) = \left\{
\begin{array}{rl}
1 & \mbox{if } x_{\widehat{S}_n}^Tv_1(\widehat{\Sigma}_{\widehat{S}_n})\geq 
\widehat{\mu}_{\widehat{S}_n}^Tv_1(\widehat{\Sigma}_{\widehat{S}_n}) \\
2 & \mbox{otherwise} 
\end{array}
\right.
$$
where $\widehat{S}_n$ is the estimated set of relevant dimensions, 
$x_{\widehat{S}_n}$ are the coordinates of $x$ restricted to the dimensions in 
$\widehat{S}_n$, and $\widehat{\mu}_{\widehat{S}_n}$ and $\widehat{\Sigma}_{\widehat{S}_n}$ are 
the sample mean and covariance of the data restricted to the dimensions in 
$\widehat{S}_n$. 
If $n\geq \max(68,4s)$, $d\geq2$ and $\alpha \leq \frac{1}{4}$, then
\begin{align*}
\sup_{\theta\in\Theta_{\lambda,s}}\bE_\theta L_\theta(\widehat{F})
\leq 603 \max\left(\frac{16\sigma^2}{\lambda^2}, 1\right)   \sqrt{\frac{s\log(ns)}{n}} + 220\frac{\sigma\sqrt{s}}{\lambda} \left(\frac{\log (nd) }{n}\right)^{\frac{1}{4}}.
\end{align*}
\end{theorem}

Next we find the lower bound.

\begin{theorem}[Lower bound]
\label{thm:lb-sparse}
Assume that $\frac{\lambda}{\sigma}\leq 0.2$, $d\geq 17$, and
that
$5 \leq s\leq \frac{d-1}{4}+1$.
Then
\begin{align*}
\inf_{\widehat{F}_n}\sup_{\theta\in\Theta_{\lambda,s}} \bE_\theta L_\theta(\widehat{F}_n) 
&\geq  \frac{1}{600} 
\min\left\{ \sqrt{\frac{8}{45}}  
\frac{\sigma^2}{\lambda^2} \sqrt{\frac{s-1}{n} \log \left(\frac{d-1}{s-1}\right) }, \frac{1}{2}  \right\} .
\end{align*}
\end{theorem}

We remark again that the constants in our bounds can be tightened, but 
the results suggest that 
$$
\hspace{-1.3cm}
\frac{\sigma}{\lambda} \left(\frac{s^2\log d}{n}\right)^{1/4} \succ
R_n \succ
\frac{\sigma^2}{\lambda^2} \sqrt{\frac{s\log d}{n}},
$$
$$
\text{or}
\quad
n = \Omega\left(\frac{s^2 \log d}{\lambda^4/\sigma^4}\right) 
\quad \text{for a constant target value of} \ R_n.
$$
In this case, we have a gap between the upper and lower bounds 
for the clustering loss. Also, the sample complexity can possibly be 
improved to scale as $s$ (instead of $s^2$) using a different method.
However, notice that the dimension 
only enters logarithmically.
If the number of relevant dimensions is small then we can expect
good rates. This provides some justification for feature selection.
We conjecture that the lower bound is tight and that the gap
could be closed by using a sparse
principal component method as in
\cite{vu2012minimax}
to find the relevant features.
However, that method is combinatorial
and so far there is no known 
computationally efficient method for implementing it with similar guarantees.

We note that the upper bound is achieved by a two-stage method
that first finds the relevant dimensions and then estimates
the clusters.
This is in contrast to the methods
described in the introduction which
do clustering and variable selection simultaneously.
This raises an interesting question:
is it always possible to achieve the minimax rate
with a two-stage procedure or are there cases
where a simultaneous method outperforms
a two-stage procedure?
Indeed, it is possible that
in the case of general covariance matrices
(non-spherical) two-stage methods might fail.
We hope to address this question in future work.

\section{Proofs of the Lower Bounds}

The lower bounds for estimation problems rely on a standard reduction from expected error to hypothesis testing that assumes the loss function is a semi-distance, which the clustering loss isn't.
However, a local triangle inequality-type bound can be shown (Proposition \ref{thm:clusterloss-triangle}).
This weaker condition can then be used to lower-bound the expected loss, as stated in Proposition \ref{thm:lowerbnd-gen} (which follows easily from Fano's inequality).

The proof techniques of the sparse and non-sparse lower bounds are almost identical.
The main difference is that in the non-sparse case, we use the Varshamov--Gilbert bound (Lemma \ref{thm:varshamov}) to construct a set of sufficiently dissimilar hypotheses, whereas in the sparse case we use an analogous result for sparse hypercubes (Lemma \ref{thm:varshamovsparse}).
See the appendix for complete proofs of all results.




\begin{lemma}[Varshamov--Gilbert bound]\label{thm:varshamov}
Let $\Omega=\{0,1\}^m$ for $m\geq8$.
There exists a subset $\{\omega_0,...,\omega_M\}\subseteq\Omega$ such that $\omega_0=(0,...,0)$, $\rho(\omega_i,\omega_j)\geq \frac{m}{8}$ for all $0\leq i<j\leq M$, and $M\geq 2^{m/8}$, where $\rho$ denotes the Hamming distance between two vectors (\cite{tsybakov2009}).
\end{lemma}

\begin{lemma}\label{thm:varshamovsparse}
Let $\Omega=\{\omega\in\{0,1\}^m:\|\omega\|_0=s\}$ for integers $m>s\geq1$ such that $s\leq m/4$.
There exist $\omega_0,...,\omega_M\in\Omega$ such that $\rho(\omega_i,\omega_j)>s/2$ for all $0\leq i<j\leq M$, and $\log(M+1)\geq \frac{s}{5}  \log\left(\frac{m}{s}\right)$ (\cite{massart2007}, Lemma 4.10).
\end{lemma}
%
\begin{prop}\label{thm:lowerbnd-gen}
Let $\theta_0,...,\theta_M\in\Theta_{\lambda}$ (or $\Theta_{\lambda,s}$), $M\geq 2$, $0<\alpha<1/8$, and $\gamma>0$.
If for all  $1\leq i\leq M$,
$
\KL(P_{\theta_i},P_{\theta_0})\leq \frac{\alpha \log M}{n},
$
and if 
$
L_{\theta_i}(\widehat{F})<\gamma
$
implies
$
L_{\theta_j}(\widehat{F}) \geq \gamma
$
for all $0\leq i\neq j \leq M$ and clusterings $\widehat{F}$,
then
$
\inf_{\widehat{F}_n}\max_{i\in[0..M]} \bE_{\theta_i} L_{\theta_i}(\widehat{F}_n) \geq 0.07\gamma.
$
\end{prop}

\begin{prop}\label{thm:clusterloss-triangle}
For any $\theta,\theta'\in\Theta_{\lambda}$, and any clustering $\widehat{F}$, let $\tau=L_\theta(\widehat{F})+\sqrt{\KL(P_\theta,P_{\theta'})/2}$.
If
$
L_\theta(F_{\theta'})+ \tau \leq 1/2,
$
then
$
L_\theta(F_{\theta'})-\tau
\leq L_{\theta'}(\widehat{F})
\leq L_\theta(F_{\theta'})+\tau.
$
\end{prop}

We will also need the following two results.
Let
$
\theta=(\mu_0-\mu/2,\mu_0+\mu/2)
$
and
$
\theta'=(\mu_0-\mu'/2,\mu_0+\mu'/2)
$
for $\mu_0,\mu,\mu'\in\bR^d$ such that $\|\mu\|=\|\mu'\|$, and let  $\cos\beta=\frac{|\mu^T\mu'|}{\|\mu\|^2}$ .

\begin{prop}\label{thm:clusterloss-bnd-new}
Let $g(x)=\phi(x)(\phi(x)-x\Phi(-x))$.
Then
$
2g\left(\frac{\|\mu\|}{2\sigma}\right) \sin\beta \cos\beta \leq
L_\theta(F_{\theta'})
\leq \frac{\tan\beta}{\pi}.
$
\end{prop}
\begin{prop}\label{thm:KL-symmetric}
Let  $\xi=\frac{\|\mu\|}{2\sigma}$.
Then
$
\KL(P_\theta,P_{\theta'}) \leq \xi^4 (1-\cos\beta)
$.
\end{prop}

{\em Proof of Theorem \ref{thm:lb-nonsparse}.}
Let $\xi=\frac{\lambda}{2\sigma}$, and define 
$
\epsilon=\min\left\{\frac{\sqrt{\log 2}}{3}\frac{\sigma^2}{\lambda}\frac{1}{\sqrt{n}}, \frac{\lambda}{4\sqrt{d-1}}\right\}.
$
Define $\lambda_0^2=\lambda^2-(d-1)\epsilon^2$.
Let $\Omega=\{0,1\}^{d-1}$.
For $\omega=(\omega(1),...,\omega(d-1))\in\Omega$, let $\mu_\omega=\lambda_0e_d+\sum_{i=1}^{d-1} (2\omega(i)-1)\epsilon e_i $ (where $\{e_i\}_{i=1}^d$ is the standard basis for $\bR^d$).
Let $\theta_\omega=\left(-\frac{\mu_\omega}{2},\frac{\mu_\omega}{2}\right)\in\Theta_\lambda$.

By Proposition \ref{thm:KL-symmetric},
$
\KL(P_{\theta_\omega},P_{\theta_\nu}) \leq \xi^4 (1-\cos\beta_{\omega,\nu})
$
where 
$
\cos\beta_{\omega,\nu}=1-\frac{2\rho(\omega,\nu)\epsilon^2}{\lambda^2}
$
and $\rho$ is the Hamming distance, so
$
\KL(P_{\theta_\omega},P_{\theta_\nu}) 
\leq \xi^4\frac{2(d-1)\epsilon^2}{\lambda^2}. 
$
By Proposition \ref{thm:clusterloss-bnd-new}, since $\cos\beta_{\omega,\nu}\geq \frac{1}{2}$,
\begin{align*}
L_{\theta_\omega}(F_{\theta_\nu})
\leq \frac{1}{\pi}\tan\beta_{\omega,\nu} 
\leq \frac{4}{\pi} \frac{\sqrt{d-1}\epsilon}{\lambda}, \mbox{ and }
\end{align*}
\begin{align*}
L_{\theta_\omega}(F_{\theta_\nu}) \geq 2g(\xi) \sin\beta_{\omega,\nu} \cos\beta_{\omega,\nu} 
\geq \sqrt{2}g(\xi) \frac{\sqrt{\rho(\omega,\nu)}\epsilon}{\lambda} 
\end{align*}
where $g(x)=\phi(x)(\phi(x)-x\Phi(-x))$.
By Lemma \ref{thm:varshamov}, there exist $\omega_0,...,\omega_M\in\Omega$ such that $M\geq 2^{(d-1)/8}$ and
$
\rho(\omega_i,\omega_j)\geq \frac{d-1}{8} 
$ 
for all $0\leq i<j\leq M$.
For simplicity of notation, let $\theta_i=\theta_{\omega_i}$ for all $i\in[0..M]$.
Then, for $i\neq j \in[0..M]$,
$$
\KL(P_{\theta_i},P_{\theta_j}) \leq \xi^4\frac{2(d-1)\epsilon^2}{\lambda^2}, \;\;
L_{\theta_i}(F_{\theta_j})
\leq \frac{4}{\pi} \frac{\sqrt{d-1}\epsilon}{\lambda} \;\; \mbox{ and } \;\;
L_{\theta_i}(F_{\theta_j}) 
\geq \frac{1}{2}g(\xi) \frac{\sqrt{d-1}\epsilon}{\lambda} .
$$
Define 
$
\gamma=\frac{1}{4}(g(\xi)-2\xi^2) \frac{\sqrt{d-1}\epsilon}{\lambda}.
$
Then for any $i\neq j\in[0..M]$, and any $\widehat{F}$ such that $L_{\theta_i}(\widehat{F})<\gamma$,
\begin{align*}
L_{\theta_i}(F_{\theta_j}) + L_{\theta_i}(\widehat{F}) + \sqrt{\frac{\KL(P_{\theta_i},P_{\theta_j})}{2}}
&<  \left(\frac{4}{\pi}  + \frac{1}{4}(g(\xi)-2\xi^2) + \xi^2 \right) \frac{\sqrt{d-1}\epsilon}{\lambda} \leq \frac{1}{2}
\end{align*}
because, for $\xi\leq 0.1$, by definition of $\epsilon$,
\begin{align*}
\left(\frac{4}{\pi}  + \frac{1}{4}(g(\xi)-2\xi^2) + \xi^2 \right) \frac{\sqrt{d-1}\epsilon}{\lambda} \leq 
2 \frac{\sqrt{d-1}\epsilon}{\lambda} \leq \frac{1}{2}.
\end{align*}
So, by Proposition \ref{thm:clusterloss-triangle},
$
L_{\theta_j}(\widehat{F})
\geq \gamma.
$
Also, 
$
\KL(P_{\theta_i},P_{\theta_0})
\leq (d-1) \xi^4\frac{2\epsilon^2}{\lambda^2} 
\leq \frac{ \log M}{9 n}
$
for all $1\leq i\leq M$,
because, by definition of $\epsilon$,
$
\xi^4\frac{2\epsilon^2}{\lambda^2} 
\leq \frac{ \log 2 }{72 n}.
$
So by Proposition \ref{thm:lowerbnd-gen} and the fact that $\xi\leq 0.1$,
\begin{align*}
\inf_{\widehat{F}_n}\max_{i\in[0..M]} \bE_{\theta_i} L_{\theta_i}(\widehat{F}_n) &\geq 0.07\gamma 
\geq \frac{1}{500}   \min\left\{\frac{\sqrt{\log 2}}{3}\frac{\sigma^2}{\lambda^2}\sqrt{\frac{d-1}{n}}, \frac{1}{4}\right\}
\end{align*}
and to complete the proof we use 
$
\sup_{\theta\in\Theta_\lambda} \bE_\theta L_\theta(\widehat{F}_n) \geq \max_{i\in[0..M]} \bE_{\theta_i} L_{\theta_i}(\widehat{F}_n)
$
for any $\widehat{F}_n$.
\hfill$\square$

{\em Proof of Theorem \ref{thm:lb-sparse}.}
For simplicity, we state this construction for $\Theta_{\lambda,s+1}$, assuming $4 \leq s\leq \frac{d-1}{4}$.
Let $\xi=\frac{\lambda}{2\sigma}$, and define 
$
\epsilon =\min\left\{ \sqrt{\frac{8}{45}} \frac{\sigma^2}{\lambda} \sqrt{\frac{1}{n} \log \left(\frac{d-1}{s}\right) }, \frac{1}{2} \frac{\lambda}{\sqrt{s}} \right\}.
$
Define $\lambda_0^2=\lambda^2-s\epsilon^2$.
Let $\Omega=\{\omega\in\{0,1\}^{d-1}:\; \|\omega\|_0=s\}$.
For $\omega=(\omega(1),...,\omega(d-1))\in\Omega$, let $\mu_\omega=\lambda_0e_d+\sum_{i=1}^{d-1} \omega(i)\epsilon e_i $ (where $\{e_i\}_{i=1}^d$ is the standard basis for $\bR^d$).
Let $\theta_\omega=\left(-\frac{\mu_\omega}{2},\frac{\mu_\omega}{2}\right)\in\Theta_{\lambda,s}$.
By Lemma \ref{thm:varshamovsparse}, there exist $\omega_0,...,\omega_M\in\Omega$ such that $\log(M+1)\geq \frac{s}{5}  \log\left(\frac{d-1}{s}\right)$ and
$
\rho(\omega_i,\omega_j)\geq \frac{s}{2}
$
for all $0\leq i<j\leq M$.
The remainder of the proof is analogous to that of Theorem \ref{thm:lb-nonsparse} with $\gamma=\frac{1}{4}(g(\xi)-\sqrt{2}\xi^2) \frac{\sqrt{s}\epsilon}{\lambda}$.
\hfill$\square$

\section{Proofs of the Upper Bounds}

Propositions \ref{thm:upper-new-meandiff} and \ref{thm:anglebound-new} below bound the error in estimating the mean and principal direction, and can be obtained using standard concentration bounds and a variant of the Davis--Kahan theorem.
Proposition \ref{thm:upper-clusterlossbnd-general} relates these errors to the clustering loss.
For the sparse case, Propositions \ref{thm:upperbound-new-variance-perdim} and \ref{thm:upperbound-new-support-recovery} bound the added error induced by the support estimation procedure.
See appendix for proof details.

\begin{prop}\label{thm:upper-new-meandiff}
Let $\theta=(\mu_0-\mu,\mu_0+\mu)$ for some $\mu_0,\mu\in\bR^d$ and $X_1,...,X_n\simiid P_\theta$.
For any $\delta>0$, we have
$
\|\mu_0-\widehat{\mu}_n\|\geq \sigma\sqrt{\frac{2\max(d,8\log\frac{1}{\delta})}{n}} + \|\mu\|\sqrt{\frac{2\log\frac{1}{\delta}}{n}} 
$
with probability at least $1-3\delta$.

\end{prop}

\begin{prop}\label{thm:anglebound-new}
Let $\theta=(\mu_0-\mu,\mu_0+\mu)$ for some $\mu_0,\mu\in\bR^d$ and $X_1,...,X_n\simiid P_\theta$ with $d>1$ and $n\geq 4d$.
Define $\cos\beta=|v_1(\sigma^2I+\mu\mu^T)^Tv_1(\widehat{\Sigma}_n)| $.
For any $0<\delta<\frac{d-1}{\sqrt{e}}$, if
$
\max\left(\frac{\sigma^2}{\|\mu\|^2},\frac{\sigma}{\|\mu\|}\right) \sqrt{\frac{\max(d,8\log\frac{1}{\delta})}{n}} \leq \frac{1}{160},
$
then with probability at least $1-12\delta-2\exp\left(-\frac{n}{20}\right)$,
$$
\sin\beta
\leq 14\max\left(\frac{\sigma^2}{\|\mu\|^2}, \frac{\sigma}{\|\mu\|}\right) \sqrt{d}  \sqrt{\frac{10}{n}\log\frac{d}{\delta}\max\left(1,\frac{10}{n}\log\frac{d}{\delta}\right)} .
$$
\end{prop}

\begin{prop}\label{thm:upper-clusterlossbnd-general}
Let $\theta=(\mu_0-\mu,\mu_0+\mu)$, and for some $x_0,v\in\bR^d$ with $\|v\|=1$, let $\widehat{F}(x) = 1$ if  $x^Tv\geq x_0^Tv$, and $2$ otherwise.
Define $\cos\beta=|v^T\mu|/\|\mu\|$.
If $|(x_0-\mu_0)^Tv|\leq \sigma\epsilon_1+\|\mu\|\epsilon_2$ for some $\epsilon_1\geq 0$ and $0\leq\epsilon_2\leq\frac{1}{4}$, and if $\sin\beta\leq\frac{1}{\sqrt{5}}$, then
\begin{align*}
L_\theta(\widehat{F})
\leq \exp\left\{-\frac{1}{2}\max\left(0,\frac{\|\mu\|}{2\sigma}-2\epsilon_1\right)^2 \right\}
 \left[ 2\epsilon_1+\epsilon_2\frac{\|\mu\|}{\sigma}  
 + 2\sin\beta\left(2\sin\beta\frac{\|\mu\|}{\sigma}+ 1\right) \right] .
\end{align*}
\end{prop}
\begin{proof}
Let $r= \left|\frac{(x_0-\mu_0)^Tv}{\cos\beta}\right|$.
Since the clustering loss is invariant to rotation and translation,
\begin{align*}
L_\theta(\widehat{F})
&\leq \frac{1}{2}\int_{-\infty}^{\infty} \frac{1}{\sigma}\phi\left(\frac{x}{\sigma}\right)
\left[\Phi\left(\frac{\|\mu\|+|x|\tan\beta + r}{\sigma}\right) -\Phi\left(\frac{\|\mu\|-|x|\tan\beta - r}{\sigma}\right) \right]dx \\
&\leq \int_{-\infty}^{\infty} \phi(x)
\left[\Phi\left(\frac{\|\mu\|}{\sigma}\right) -\Phi\left(\frac{\|\mu\|- r}{\sigma}-|x|\tan\beta\right) \right]dx .
\end{align*}
Since $\tan\beta\leq \frac{1}{2}$ and $\epsilon_2\leq\frac{1}{4}$, we have $r\leq 2\sigma\epsilon_1+2\|\mu\|\epsilon_2$, and 
$
\Phi\left(\frac{\|\mu\|}{\sigma}\right) -\Phi\left(\frac{\|\mu\| - r}{\sigma}\right) 
\leq  2\left(\epsilon_1+\epsilon_2\frac{\|\mu\|}{\sigma} \right)\phi\left(\max\left(0,\frac{\|\mu\|}{2\sigma} - 2\epsilon_1\right)\right).
$
Defining $A=\left|\frac{\|\mu\|- r}{\sigma}\right|$, 
\begin{align*}
&\int_{-\infty}^{\infty} \phi(x)
\left[\Phi\left(\frac{\|\mu\|-r}{\sigma}\right) -\Phi\left(\frac{\|\mu\|- r}{\sigma}-|x|\tan\beta\right) \right]dx  
\leq 
2\int_{0}^{\infty} \int_{A-x\tan\beta}^{A}  \phi(x) \phi(y) dydx \\
&=2\int_{-A\sin\beta}^{\infty} \int_{A\cos\beta}^{A\cos\beta+(u+A\sin\beta)\tan\beta}  \phi(u) \phi(v) dudv 
\leq 2\phi\left(A\right)\tan\beta \left(A\sin\beta + 1\right) \\
&\leq 2\phi\left(\max\left(0,\frac{\|\mu\|}{2\sigma}-2\epsilon_1\right)\right)\tan\beta \left(\left(2\frac{\|\mu\|}{\sigma}+2\epsilon_1\right)\sin\beta + 1\right)
\end{align*}
where we used $u=x\cos\beta-y\sin\beta$ and $v=x\sin\beta+y\cos\beta$ in the second step.
The bound now follows easily.
\end{proof}

{\em Proof of Theorem \ref{thm:u5-ns-upper}.}
Using Propositions \ref{thm:upper-new-meandiff} and \ref{thm:anglebound-new} with $\delta=\frac{1}{\sqrt{n}}$, Proposition \ref{thm:upper-clusterlossbnd-general}, and the fact that  $(C+x)\exp(-\max(0,x-4)^2/8)\leq(C+6)\exp(-\max(0,x-4)^2/10)$ for all $C,x>0$,
\begin{align*}
\bE_\theta L_\theta(\widehat{F})
&\leq 600 \max\left(\frac{4\sigma^2}{\lambda^2}, 1\right)   \sqrt{\frac{d\log(nd)}{n}}
\end{align*}
(it is easy to verify that the bounds are decreasing with $\|\mu\|$, so we use $\|\mu\|=\frac{\lambda}{2}$ to bound the supremum).
In the $d=1$ case Proposition \ref{thm:anglebound-new} need not be applied, since the principal directions agree trivially.
The bound for $\frac{\lambda}{\sigma} \geq 2\max(80,14\sqrt{5d})$ can be shown similarly, using $\delta=\exp\left(-\frac{n}{32}\right)$.
\hfill$\square$

\begin{prop}\label{thm:upperbound-new-variance-perdim}
Let $\theta=(\mu_0-\mu,\mu_0+\mu)$ for some $\mu_0,\mu\in\bR^d$ and $X_1,...,X_n\simiid P_\theta$.
For any $0<\delta<\frac{1}{\sqrt{e}}$ such that $\sqrt{\frac{6\log \frac{1}{\delta} }{n}}\leq \frac{1}{2}$, with probability at least $1-6d\delta$, for all $i\in[d]$,
$$
|\widehat{\Sigma}_n(i,i) - (\sigma^2+\mu(i)^2)| \leq \sigma^2 \sqrt{\frac{6\log \frac{1}{\delta} }{n}} + 2\sigma|\mu(i)|\sqrt{\frac{2\log\frac{1}{\delta}}{n}} + (\sigma+|\mu(i)|)^2\frac{2\log\frac{1}{\delta}}{n}.
$$
\end{prop}

\begin{prop}\label{thm:upperbound-new-support-recovery}
Let $\theta=(\mu_0-\mu,\mu_0+\mu)$ for some $\mu_0,\mu\in\bR^d$ and $X_1,...,X_n\simiid P_\theta$.
Define
$$
S(\theta)=\{i\in[d]: \mu(i)\neq 0\} \;\;  \mbox{ and } \;\;
\widetilde{S}(\theta)=\{i\in[d]: |\mu(i)|\geq 4\sigma\sqrt{\alpha}\}.
$$
Assume that $n\geq 1$, $d\geq2$, and $\alpha \leq \frac{1}{4}$.
Then $\widetilde{S}(\theta) \subseteq \widehat{S}_n \subseteq S(\theta)$ with probability at least $1-\frac{6}{n}$.
\end{prop}
\begin{proof}
By Proposition \ref{thm:upperbound-new-variance-perdim}, with probability at least $1-\frac{6}{n}$,
$$
|\widehat{\Sigma}_n(i,i) - (\sigma^2+\mu(i)^2)| \leq \sigma^2 \sqrt{\frac{6\log (nd) }{n}} + 2\sigma|\mu(i)|\sqrt{\frac{2\log(nd)}{n}} + (\sigma+|\mu(i)|)^2\frac{2\log(nd)}{n}
$$
for all $i\in[d]$.
Assume the above event holds.
If $S(\theta)=[d]$, then of course $\widehat{S}_n \subseteq S(\theta)$.
Otherwise, for $i\notin S(\theta)$, we have
$
(1-\alpha)\sigma^2 \leq \widehat{\Sigma}_n(i,i)  \leq (1+\alpha)\sigma^2
$,
so it is clear that $\widehat{S}_n \subseteq S(\theta)$.
The remainder of the proof is trivial if $\widetilde{S}(\theta) =\emptyset$ or $S(\theta)=\emptyset$.
Assume otherwise.
For any $i \in S(\theta)$,
\begin{align*}
\widehat{\Sigma}_n(i,i) 
&\geq (1-\alpha)\sigma^2 + \left(1-\frac{2\log(nd)}{n}\right) \mu(i)^2  - 2\alpha \sigma|\mu(i)|.
\end{align*}
By definition, $|\mu(i)|\geq 4\sigma\sqrt{\alpha}$ for all $i \in \widetilde{S}(\theta)$, so
$
\frac{(1+\alpha)^2}{1-\alpha} \sigma^2\leq 
\widehat{\Sigma}_n(i,i) 
$
and $i \in \widehat{S}_n$ (we ignore strict equality above as a measure $0$ event), i.e. $\widetilde{S}(\theta)\subseteq \widehat{S}_n$, which concludes the proof.
\end{proof}

{\em Proof of Theorem \ref{thm:upperbound-new-support-recovery-plus-estimation}.}
Define
$
S(\theta)=\{i\in[d]: \mu(i)\neq 0\}
$
and
$
\widetilde{S}(\theta)=\{i\in[d]: |\mu(i)|\geq 4\sigma\sqrt{\alpha}\}
$.
Assume $\widetilde{S}(\theta) \subseteq \widehat{S}_n \subseteq S(\theta)$ (by Proposition \ref{thm:upperbound-new-support-recovery}, this holds with probability at least $1-\frac{6}{n}$).
If $\widetilde{S}(\theta)=\emptyset$, then we simply have $\bE_\theta L_\theta(\widehat{F}_n) \leq \frac{1}{2}$.

Assume $\widetilde{S}(\theta)\neq\emptyset$.
Let 
$\cos\widehat{\beta} = |v_1(\widehat{\Sigma}_{\widehat{S}_n})^Tv_1(\Sigma)|,$
$\cos\widetilde{\beta} = |v_1(\Sigma_{\widehat{S}_n})^Tv_1(\Sigma)|$, and 
$\cos\beta = |v_1(\widehat{\Sigma}_{\widehat{S}_n})^Tv_1(\Sigma_{\widehat{S}_n})|$
where $\Sigma=\sigma^2I+\mu\mu^T$, and for simplicity we define $\widehat{\Sigma}_{\widehat{S}_n}$ and $\Sigma_{\widehat{S}_n}$ to be the same as $\widehat{\Sigma}_n$ and $\Sigma$ in $\widehat{S}_n$, respectively, and $0$ elsewhere.
Then
$\sin\widehat{\beta}\leq \sin\widetilde{\beta} + \sin\beta$, and
\begin{align*}
\sin\widetilde{\beta} &= \frac{\|\mu-\mu_{\widehat{S}(\theta)}\|}{\|\mu\|} 
\leq \frac{\|\mu-\mu_{\widetilde{S}(\theta)}\|}{\|\mu\|}  
\leq \frac{4\sigma\sqrt{\alpha}\sqrt{|S(\theta)|-|\widetilde{S}(\theta)}|}{\|\mu\|}  
\leq 8\frac{\sigma\sqrt{s\alpha}}{\lambda}.
\end{align*}

Using the same argument as the proof of Theorem \ref{thm:u5-ns-upper}, as long as the above bound is smaller than $\frac{1}{2\sqrt{5}}$,
\begin{align*}
\bE_\theta L_\theta(\widehat{F})
&\leq 600 \max\left(\frac{\sigma^2}{\left(\frac{\lambda}{2}-4\sigma\sqrt{s\alpha}\right)^2}, 1\right)   \sqrt{\frac{s\log(ns)}{n}} + 104\frac{\sigma\sqrt{s\alpha}}{\lambda} + \frac{3}{n} .
\end{align*}
Using the fact $L_\theta(\widehat{F})\leq\frac{1}{2}$ always, and that $\alpha\leq\frac{1}{4}$ implies $\frac{\log(nd)}{n}\leq 1$, the bound follows.
\hfill$\square$

\section{Conclusion}

We have provided minimax lower
and upper bounds for estimating high dimensional mixtures.
The bounds show explicitly how the 
statistical difficulty of the problem
depends on dimension $d$, sample size $n$,
separation $\lambda$ and sparsity level $s$.

For clarity,
we have focused on the special case
where there are two spherical components and the mixture weights are equal.
In future work, we plan to extend the results
to general mixtures of $k$ Gaussians.

One of our motivations for this work
is the recent interest in variable selection methods
to facilitate clustering in high dimensional problems.
Existing methods such as
\citet{pan2007penalized,
witten2010framework, raftery2006variable,sun2012regularized}
and \citet{guo2010pairwise}
provide promising numerical evidence
that variable selection does improve
high dimensional clustering.
Our results provide some theoretical basis for this idea.

However, there is a gap between
the results in this paper and the methodology papers mentioned above.
Indeed, as of now, there is no rigorous proof that
the methods in those papers
outperform a two stage approach where
the first stage screens for relevant features
and the second stage applies standard clustering methods
on the features extracted from the first stage.
We conjecture that there are conditions under which
simultaneous feature selection and clustering
outperforms the two stage approach.
Settling this questions will
require the aforementioned extension of our results
to the general mixture case.


\bibliographystyle{plainnat}
\bibliography{nips2013}

%
%
%
%
%
%
%
%
%
%

\section{Notation}

For $\theta=(\mu_1,\mu_2)\in\bR^{2\times d}$, define
$$p_\theta(x)=\frac{1}{2} f(x;\mu_1,\sigma^2I)+\frac{1}{2} f(x;\mu_2,\sigma^2I),$$ 
where $f(\cdot;\mu,\Sigma)$ is the density of $\cN(\mu,\Sigma)$, $\sigma>0$ is a fixed
constant.
Let $P_\theta$ denote the probability measure corresponding to $p_\theta$.
We consider two classes $\Theta$ of parameters:
$$
\Theta_{\lambda}=\left\{(\mu_1,\mu_2): \|\mu_1-\mu_2\|\geq\lambda\right\}
$$
$$
\Theta_{\lambda,s}=
\left\{(\mu_1,\mu_2): \|\mu_1-\mu_2\|\geq\lambda,\;\|\mu_1-\mu_2\|_0\leq s\right\}\subseteq \Theta_{\lambda}.
$$

Throughout this document, $\phi$ and $\Phi$ denote the standard normal 
density and distribution functions.

For a mixture with parameter $\theta$, the Bayes optimal classification, that is, assignment 
of a point $x \in \bR^d$ to the correct mixture component,
is given by the function 
$$
F_\theta(x)=\argmax\limits_{i\in\{1,2\}} f(x;\mu_i,\sigma^2I).
$$
Given any other candidate assignment function 
$F:\bR^d\rightarrow\{1,2\}$, we define the loss incurred by $F$ as
$$
L_\theta(F)=\min\limits_\pi P_\theta(\{x:F_\theta(x)\neq \pi(F(x))\})
$$ where 
the minimum is over all permutations $\pi:\{1,2\}\rightarrow\{1,2\}$. 

For $X_1, \dots, X_n\simiid P_\theta$, let $\widehat{\mu}_n$ and $\widehat{\Sigma}_n$ be the mean and covariance of the corresponding empirical distribution.

Also, for a matrix $B$, $v_i(B)$ and $\lambda_i(B)$ are the $i$'th eigenvector and eigenvalue of $B$ (assuming $B$ is symmetric), arranged so that $\lambda_i(B)\geq\lambda_{i+1}(B)$, and $\|B\|_2$ is the spectral norm.

\section{Upper bounds}

\subsection{Standard concentration bounds}

\subsubsection{Concentration bounds for estimating the mean}

\begin{prop}\label{thm2:concentration-new-chisq}
Let $X\sim\chi_d^2$.
Then for any $\epsilon>0$,
 $$
 \bP(X>(1+\epsilon)d) \leq \exp\left\{-\frac{d}{2}\left(\epsilon - \log(1+\epsilon)\right)\right\}.
 $$
 If $\epsilon<1$, then
 $$
 \bP(X<(1-\epsilon)d) \leq \exp\left\{\frac{d}{2}(\epsilon + \log (1-\epsilon)) \right\}.
 $$
\end{prop}

\iftoggle{detailed3} {
\begin{proof}
 Since $\bE e^{tX} = (1-2t)^{-\frac{d}{2}}$ for $0<t<\frac{1}{2}$,
\begin{align*}
\bP(X>(1+\epsilon)d)
&=\bP(e^{tX}>e^{t(1+\epsilon)d}) \\
&\leq e^{-t(1+\epsilon)d} (1-2t)^{-\frac{d}{2}} \\
&= \exp\left[-t(1+\epsilon)d + \frac{d}{2} \log \frac{1}{1-2t} \right].
\end{align*}
To minimize the right hand side, we differentiate the exponent with respect to $t$ to obtain the equation
\begin{align*}
-(1+\epsilon)d + \frac{d}{1-2t} = 0
\end{align*}
which can be satisfied by setting $t=\frac{1}{2}\left(1-\frac{1}{1+\epsilon}\right)<\frac{1}{2}$ (it is easy to verify that this is a global minimum).
Using this value for $t$, the first bound follows.

Also, for $t>0$ and $\epsilon<1$,
\begin{align*}
\bP(X<(1-\epsilon)d)
&=\bP(e^{-tX}>e^{-t(1-\epsilon)d}) \\
&\leq e^{t(1-\epsilon)d} (1+2t)^{-\frac{d}{2}} \\
&= \exp\left[t(1-\epsilon)d - \frac{d}{2} \log (1+2t) \right]
\end{align*}
and setting $t=\frac{1}{2}\left(\frac{1}{1-\epsilon} - 1\right)$,
\begin{align*}
\bP(X<(1-\epsilon)d)
&\leq \exp\left[\frac{d}{2}(\epsilon + \log (1-\epsilon)) \right].
\end{align*}
\end{proof}} {}

\begin{prop}\label{thm2:concentration-new-gaussianmean}
Let $Z_1,...,Z_n\simiid\cN(0,I_d)$.
Then for any $\epsilon>0$,
$$
 \bP\left(\left\|\frac{1}{n}\sum_{i=1}^nZ_i\right\|\geq \sqrt{\frac{(1+\epsilon)d}{n}}\right)\leq \exp\left\{-\frac{d}{2}\left(\epsilon - \log(1+\epsilon)\right)\right\}.
 $$
\end{prop}

\iftoggle{detailed3} {
\begin{proof}
Using Proposition \ref{thm2:concentration-new-chisq},
  \begin{align*}
 \bP\left(\left\|\frac{1}{n}\sum_{i=1}^nZ_i\right\|\geq \sqrt{\frac{(1+\epsilon)d}{n}}\right)
 &= \bP\left(\left\|\frac{1}{\sqrt{n}} \sum_{i=1}^n Z_i \right\|^2\geq (1+\epsilon)d\right)\\
 &= \bP\left(X\geq (1+\epsilon)d\right) \\
 &\leq \exp\left\{-\frac{d}{2}(\epsilon - \log(1+\epsilon))\right\}
 \end{align*}
 where  $X\sim\chi_d^2$.
\end{proof}} {}

\subsubsection{Concentration bounds for estimating principal direction}

\begin{prop}\label{thm2:concentration-new-wishart-spectral}
Let $Z_1,...,Z_n\simiid\cN(0,I_d)$ and $\delta>0$.
If $n\geq d$ then with probability at least $1-3\delta$,
\begin{align*}
\|\widehat{\Sigma}_n - I_d\|_2 \leq&  3 \left(1+\sqrt{\frac{2\log\frac{1}{\delta}}{d}}\right)\sqrt{\frac{d}{n}} \max\left(1,\left(1+\sqrt{\frac{2\log\frac{1}{\delta}}{d}}\right)\sqrt{\frac{d}{n}}\right) \\
&+ \left(1+\sqrt{\frac{8\log\frac{1}{\delta}}{d}\max\left(1,\frac{8\log\frac{1}{\delta}}{d}\right)}\right) \frac{d}{n} 
\end{align*}
where $\widehat{\Sigma}_n$ is the empirical covariance of $Z_i$.
\end{prop}

\iftoggle{detailed3} {
\begin{proof}
Let $\overline{Z}_n=\frac{1}{n}\sum_{i=1}^nZ_i$.
Then
\begin{align*}
\|\widehat{\Sigma}_n - I_d\|_2
&= \left\|\frac{1}{n}\sum_{i=1}^n Z_iZ_i^T - I_d - \overline{Z}_n \overline{Z}_n^T\right\|_2 \\
&\leq \left\|\frac{1}{n}\sum_{i=1}^n Z_iZ_i^T -I_d\right\|_2 + \left\|\overline{Z}_n\right\|^2.
\end{align*}
It is well known that for any $\epsilon_1>0$,
$$
\bP\left(\left\|\frac{1}{n}\sum_{i=1}^n Z_iZ_i^T-I_d\right\|_2 \geq 3 \left(1+\epsilon_1\right)\sqrt{\frac{d}{n}} \max\left(1,(1+\epsilon_1)\sqrt{\frac{d}{n}}\right) \right) \leq 2\exp\left\{-\frac{d\epsilon_1^2}{2}\right\}.
$$
Using this along with Proposition \ref{thm2:concentration-new-gaussianmean}, we have for any $\epsilon_2>0$,
\begin{align*}
&\bP\left(\|\widehat{\Sigma}_n - I_d\|_2\geq  3 \left(1+\epsilon_1\right)\sqrt{\frac{d}{n}} \max\left(1,(1+\epsilon_1)\sqrt{\frac{d}{n}}\right) +  \frac{(1+\epsilon_2)d}{n} \right) \\
&\leq 2\exp\left\{-\frac{d\epsilon_1^2}{2}\right\} + \exp\left\{-\frac{d}{2}\left(\epsilon_2 - \log(1+\epsilon_2)\right)\right\}.
\end{align*}
Setting $\epsilon_1=\sqrt{\frac{2\log\frac{1}{\delta}}{d}}$,
\begin{align*}
&\bP\left(\|\widehat{\Sigma}_n - I_d\|_2\geq 3 \left(1+\sqrt{\frac{2\log\frac{1}{\delta}}{d}}\right)\sqrt{\frac{d}{n}} \max\left(1,\left(1+\sqrt{\frac{2\log\frac{1}{\delta}}{d}}\right)\sqrt{\frac{d}{n}}\right) +  \frac{(1+\epsilon_2)d}{n} \right) \\
&\leq 2\delta + \exp\left\{-\frac{d}{2}\left(\epsilon_2 - \log(1+\epsilon_2)\right)\right\} \\
&\leq 2\delta + \exp\left\{-\frac{d}{8}\epsilon_2\min(1,\epsilon_2)\right\} 
\end{align*}
and, setting $\epsilon_2=\sqrt{\frac{8\log\frac{1}{\delta}}{d}\max\left(1,\frac{8\log\frac{1}{\delta}}{d}\right)}$, with probability at least $1-3\delta$,
\begin{align*}
\|\widehat{\Sigma}_n - I_d\|_2 \leq&  3 \left(1+\sqrt{\frac{2\log\frac{1}{\delta}}{d}}\right)\sqrt{\frac{d}{n}} \max\left(1,\left(1+\sqrt{\frac{2\log\frac{1}{\delta}}{d}}\right)\sqrt{\frac{d}{n}}\right) \\
&+ \left(1+\sqrt{\frac{8\log\frac{1}{\delta}}{d}\max\left(1,\frac{8\log\frac{1}{\delta}}{d}\right)}\right) \frac{d}{n} .
\end{align*}
\end{proof} } {}

\begin{prop}\label{thm2:concentration-new-prodnormal}
Let $X_1,Y_1,...,X_n,Y_n\simiid\cN(0,1)$.
Then for any $\epsilon>0$,
\begin{align*}
\bP\left(\left|\frac{1}{n}\sum_{i=1}^n X_iY_i\right|>\frac{\epsilon}{2}\right)
&\leq 2 \exp\left\{-\frac{n\epsilon\min(1,\epsilon)}{10}\right\}.
\end{align*}
\end{prop}
\iftoggle{detailed3}{
\begin{proof}
Let $Z=XY$ where $X,Y\simiid\cN(0,1)$.
Then for any $t$ such that $|t|<1$,
$$
\bE e^{tZ} = \frac{1}{\sqrt{1-t^2}}.
$$
So for $0<t<1$,
\begin{align*}
\bP\left(\frac{1}{n}\sum_{i=1}^nX_iY_i>\epsilon\right)
&= \bP\left(\exp\left\{\sum_{i=1}^n tX_iY_i\right\}>\exp(n\epsilon t)\right) \\
&\leq \bE\left(\exp\left\{\sum_{i=1}^n tX_iY_i\right\}\right)\exp(-n\epsilon t) \\
&= (\bE\exp(tX_iY_i))^n\exp(-n\epsilon t) \\
&= (1-t^2)^{-\frac{n}{2}}\exp(-n\epsilon t) \\
&= \exp\left\{-\frac{n}{2}\left(2\epsilon t+\log(1-t^2)\right)\right\}.
\end{align*}
The bound is minimized by $t = \frac{1}{2\epsilon}\left( \sqrt{ 1 + 4\epsilon^2 } - 1 \right)<1$, so
\begin{align*}
\bP\left(\frac{1}{n}\sum_{i=1}^nX_iY_i>\epsilon\right)
&\leq \exp\left\{-\frac{n}{2}h(2\epsilon)\right\}
\end{align*}
where 
$$
h(u) = \left( \sqrt{ 1 + u^2 } - 1 \right) +\log\left(1-\frac{1}{u^2}\left( \sqrt{ 1 + u^2 } - 1 \right)^2\right).
$$
Since $h(u)\geq\frac{u}{5}\min(1,u)$,
\begin{align*}
\bP\left(\frac{1}{n}\sum_{i=1}^nX_iY_i>\epsilon\right)
&\leq \exp\left\{-\frac{n}{2}\frac{2\epsilon}{5}\min\left(1,2\epsilon\right)\right\}
\end{align*}
and the proof is complete by noting that the distribution of $X_iY_i$ is symmetric.
\end{proof} }{}

\subsection{Davis--Kahan}

%

\begin{lemma}\label{thm2:DK-tighter-one}
Let $A,E\in\bR^{d\times d}$ be symmetric matrices, and $u\in\bR^{d-1}$ such that
$$
u_i=v_{i+1}(A)^TEv_1(A).
$$
If $\lambda_1(A)-\lambda_2(A)>0$ and
$$
\|E\|_2\leq\frac{\lambda_1(A)-\lambda_2(A)}{5}
$$
then
$$
\sqrt{1-(v_1(A)^Tv_1(A+E))^2} \leq \frac{4\|u\|}{\lambda_1(A)-\lambda_2(A)} 
$$
(Corollary 8.1.11 of \cite{golub1996}).
\end{lemma}

\subsection{Bounding error in estimating the mean}

\begin{prop}\label{thm2:upper-new-meandiff}
Let $\theta=(\mu_0-\mu,\mu_0+\mu)$ for some $\mu_0,\mu\in\bR^d$ and $X_1,...,X_n\simiid P_\theta$.
For any $\delta>0$,
\begin{align*}
\bP\left(\|\mu_0-\widehat{\mu}_n\|\geq \sigma\sqrt{\frac{2\max(d,8\log\frac{1}{\delta})}{n}} + \|\mu\|\sqrt{\frac{2\log\frac{1}{\delta}}{n}} \right)
\leq 3\delta.
\end{align*}
\end{prop}
\iftoggle{detailed3} {
\begin{proof}
Let $Z_1,...,Z_n\simiid \cN(0,I)$ and $Y_1,...,Y_n$ i.i.d. such that $\bP(Y_i=-1)=\bP(Y_i=1)=\frac{1}{2}$.
Then for any $\epsilon_1,\epsilon_2>0$,
\begin{align*}
&\bP\left(\|\mu_0-\widehat{\mu}_n\|\geq \sigma\sqrt{\frac{(1+\epsilon_1)d}{n}} + \|\mu\|\epsilon_2 \right)\\
&= \bP\left(\left\|\mu_0-\frac{1}{n}\sum_{i=1}^d(\sigma Z_i + \mu_0 + \mu Y_i)\right\|\geq \sigma\sqrt{\frac{(1+\epsilon_1)d}{n}} + \|\mu\|\epsilon_2 \right)\\
&= \bP\left(\left\|\sigma\frac{1}{n}\sum_{i=1}^d Z_i + \mu \frac{1}{n}\sum_{i=1}^d Y_i\right\|\geq \sigma\sqrt{\frac{(1+\epsilon_1)d}{n}} + \|\mu\|\epsilon_2 \right) \\
&\leq \bP\left(\sigma\left\|\frac{1}{n}\sum_{i=1}^d Z_i\right\| + \|\mu\| \left|  \frac{1}{n}\sum_{i=1}^d Y_i\right|\geq \sigma\sqrt{\frac{(1+\epsilon_1)d}{n}} + \|\mu\|\epsilon_2 \right) \\
&\leq \bP\left(\left\|\frac{1}{n}\sum_{i=1}^d Z_i\right\| \geq \sqrt{\frac{(1+\epsilon_1)d}{n}} \right) + 
\bP\left(\left|  \frac{1}{n}\sum_{i=1}^d Y_i\right|\geq \epsilon_2 \right) \\
&\leq \exp\left\{-\frac{d}{2}\left(\epsilon_1 - \log(1+\epsilon_1)\right)\right\} + 2\exp\left\{-\frac{n\epsilon_2^2}{2}\right\}
\end{align*}
where the last step is using Hoeffding's inequality and Proposition \ref{thm2:concentration-new-gaussianmean}.
Setting $\epsilon_2=\sqrt{\frac{2\log\frac{1}{\delta}}{n}}$,
\begin{align*}
&\bP\left(\|\mu_0-\widehat{\mu}_n\|\geq \sigma\sqrt{\frac{(1+\epsilon_1)d}{n}} + \|\mu\|\sqrt{\frac{2\log\frac{1}{\delta}}{n}} \right)\\
&\leq \exp\left\{-\frac{d}{2}\left(\epsilon_1 - \log(1+\epsilon_1)\right)\right\} + 2\delta.
\end{align*}
Since $\epsilon_1 - \log(1+\epsilon_1)\geq \frac{\epsilon_1}{4}\min(1,\epsilon_1)$,
$$
\exp\left\{-\frac{d}{2}\left(\epsilon_1 - \log(1+\epsilon_1)\right)\right\} \leq \exp\left\{-\frac{d}{8}\epsilon_1\min(1,\epsilon_1)\right\}.
$$
Setting
$$
\epsilon_1=\sqrt{\frac{8\log\frac{1}{\delta}}{d}\max\left(1,\frac{8\log\frac{1}{\delta}}{d}\right)},
$$
we have
\begin{align*}
\bP\left(\|\mu_0-\widehat{\mu}_n\|\geq \sigma\sqrt{\frac{d}{n}\left(1+\sqrt{\frac{8\log\frac{1}{\delta}}{d}\max\left(1,\frac{8\log\frac{1}{\delta}}{d}\right)}\right)} + \|\mu\|\sqrt{\frac{2\log\frac{1}{\delta}}{n}} \right)
\leq 3\delta
\end{align*}
and the bound follows.
\end{proof}} {}

\subsection{Bounding error in estimating principal direction}

\begin{prop}\label{thm2:concentration-new-gmmixture-spectral}
Let $\theta=(\mu_0-\mu,\mu_0+\mu)$ for some $\mu_0,\mu\in\bR^d$ and $X_1,...,X_n\simiid P_\theta$.
If $n\geq d$ then for any $\delta,\delta_1>0$, with probability at least $1-5\delta-2\delta_1$,
\begin{align*}
&\|\widehat{\Sigma}_n - (\sigma^2 I_d+\mu\mu^T)\|_2 \\
&\leq 3\sigma^2 \left(1+\sqrt{\frac{2\log\frac{1}{\delta}}{d}}\right)\sqrt{\frac{d}{n}} \max\left(1,\left(1+\sqrt{\frac{2\log\frac{1}{\delta}}{d}}\right)\sqrt{\frac{d}{n}}\right) \\
&+\sigma^2 \left(1+\sqrt{\frac{8\log\frac{1}{\delta}}{d}\max\left(1,\frac{8\log\frac{1}{\delta}}{d}\right)}\right) \frac{d}{n} \\
&+ 4\sigma  \|\mu\|  \sqrt{\left(1+\sqrt{\frac{8\log\frac{1}{\delta}}{d}\max\left(1,\frac{8\log\frac{1}{\delta}}{d}\right)}\right)\frac{d}{n}} 
+ \frac{2\|\mu\|^2\log\frac{1}{\delta_1}}{n} 
\end{align*}
where $\widehat{\Sigma}_n$ is the empirical covariance of $X_i$.
\end{prop}

\iftoggle{detailed3} {
\begin{proof}
We can express $X_i$ as $X_i=\sigma Z_i + \mu Y_i+\mu_0$ where $Z_1,...,Z_n\simiid \cN(0,I_d)$ and $Y_1,...,Y_n$ i.i.d. such that $\bP(Y_i=-1)=\bP(Y_i=1)=\frac{1}{2}$.
Then
\begin{align*}
\widehat{\Sigma}_n - (\sigma^2 I_d+\mu\mu^T)
&= \sigma^2 (\widehat{\Sigma}_n^Z - I_d)   - \mu\mu^T \overline{Y}^2 \\
&+ \sigma  \left(  \frac{1}{n}\sum_{i=1}^n Y_iZ_i  - \overline{Y}\overline{Z} \right) \mu^T \\
&+  \sigma \mu  \left(  \frac{1}{n}\sum_{i=1}^n Y_iZ_i  - \overline{Y}\overline{Z} \right) ^T
\end{align*}
where $\widehat{\Sigma}_n^Z$ is the empirical covariance of $Z_i$ and $\overline{Y}$ and $\overline{Z}$ are the empirical means of $Y_i$ and $Z_i$.
So
\begin{align*}
\|\widehat{\Sigma}_n - (\sigma^2 I_d+\mu\mu^T)\|_2
&\leq \sigma^2 \|\widehat{\Sigma}_n^Z - I_d\|_2  + \|\mu\|^2 \overline{Y}^2 \\
&+ 2\sigma  \|\mu\| \left( \left\|  \frac{1}{n}\sum_{i=1}^n Y_iZ_i \right\| + |\overline{Y}|\|\overline{Z} \| \right).
\end{align*}
By Hoeffding's inequality, 
\begin{align*}
\bP\left( \|\mu\|^2 \overline{Y}^2 \geq  \frac{2\|\mu\|^2\log\frac{1}{\delta_1}}{n} \right) \leq 2 \delta_1.
\end{align*}
Since $|\overline{Y}|\leq 1$ and since $Y_i Z_i$ has the same distribution as $Z_i$, by Proposition \ref{thm2:concentration-new-gaussianmean}, for any $\epsilon>0$,
\begin{align*}
&\bP\left( 2\sigma  \|\mu\| \left( \left\|  \frac{1}{n}\sum_{i=1}^n Y_iZ_i \right\| + |\overline{Y}|\|\overline{Z} \| \right) \geq 4\sigma  \|\mu\|  \sqrt{\frac{(1+\epsilon)d}{n}} \right) \\
&\leq 2  \exp\left\{-\frac{d}{2}\left(\epsilon - \log(1+\epsilon)\right)\right\} 
\leq 2  \exp\left\{-\frac{d}{8}\epsilon\min(1,\epsilon)\right\} .
\end{align*}
Setting 
$$
\epsilon = \sqrt{\frac{8\log\frac{1}{\delta}}{d}\max\left(1,\frac{8\log\frac{1}{\delta}}{d}\right)}
$$
we have
\begin{align*}
&\bP\left( 2\sigma  \|\mu\| \left( \left\|  \frac{1}{n}\sum_{i=1}^n Y_iZ_i \right\| + |\overline{Y}|\|\overline{Z} \| \right) \geq 4\sigma  \|\mu\|  \sqrt{\left(1+\sqrt{\frac{8\log\frac{1}{\delta}}{d}\max\left(1,\frac{8\log\frac{1}{\delta}}{d}\right)}\right)\frac{d}{n}} \right) \\
& \leq 2 \delta .
\end{align*}
Finally, by Proposition \ref{thm2:concentration-new-wishart-spectral}, with probability at least $1-3\delta$,
\begin{align*}
\sigma^2\|\widehat{\Sigma}_n^Z - I_d\|_2 \leq&  3\sigma^2 \left(1+\sqrt{\frac{2\log\frac{1}{\delta}}{d}}\right)\sqrt{\frac{d}{n}} \max\left(1,\left(1+\sqrt{\frac{2\log\frac{1}{\delta}}{d}}\right)\sqrt{\frac{d}{n}}\right) \\
&+\sigma^2 \left(1+\sqrt{\frac{8\log\frac{1}{\delta}}{d}\max\left(1,\frac{8\log\frac{1}{\delta}}{d}\right)}\right) \frac{d}{n} 
\end{align*}
and we complete the proof by combining the three bounds.
\end{proof} } {}

\begin{prop}\label{thm2:concentration-new-offdiagonal}
Let $\theta=(\mu_0-\mu,\mu_0+\mu)$ for some $\mu_0,\mu\in\bR^d$ and $X_1,...,X_n\simiid P_\theta$.
If $n\geq d>1$ then for any $0<\delta\leq \frac{1}{\sqrt{e}}$ and $i\in[2..d]$, with probability at least $1-7\delta$,
\begin{align*}
&\left|v_i(\sigma^2I+\mu\mu^T)^T(\widehat{\Sigma}_n-(\sigma^2I+\mu\mu^T))v_1(\sigma^2I+\mu\mu^T)\right| \\
&\leq \sigma^2 \frac{1}{2}\sqrt{\frac{10\log\frac{1}{\delta}}{n}\max\left(1,\frac{10\log\frac{1}{\delta}}{n}\right)} 
+ \sigma\|\mu\| \sqrt{\frac{2\log\frac{1}{\delta}}{n}}  + (\sigma^2+  \sigma\|\mu\|) \frac{2\log\frac{1}{\delta}}{n}.
\end{align*}
\end{prop}

\iftoggle{detailed3}{
\begin{proof}
Let $Z_1,W_1,...,Z_n,W_n\simiid \cN(0,1)$ and $Y_1,...,Y_n$ i.i.d. such that $\bP(Y_i=-1)=\bP(Y_i=1)=\frac{1}{2}$.
It is easy to see that the quantity of interest is equal in distribution to
$$
\left|\frac{1}{n}\sum_{j=1}^n(\sigma Z_i - \sigma \overline{Z})(\sigma W_i - \sigma \overline{W} + \|\mu\|Y_i-\|\mu\|\overline{Y})\right|
$$
where $\overline{Z},\overline{W},\overline{Y}$ are the respective empirical means.
Moreover,
\begin{align*}
&\left|\frac{1}{n}\sum_{j=1}^n(\sigma Z_i - \sigma \overline{Z})(\sigma W_i - \sigma \overline{W} + \|\mu\|Y_i-\|\mu\|\overline{Y})\right|\\
&\leq \sigma^2 \left|\frac{1}{n}\sum_{j=1}^n Z_i W_i \right|+ \sigma^2 \left|\overline{Z}\right|\left|\overline{W}\right| + \sigma\|\mu\| \left| \frac{1}{n}\sum_{j=1}^n Z_iY_i \right| + \sigma\|\mu\| \left|\overline{Z}\right|\left|\overline{Y}\right|.
\end{align*}
From Proposition \ref{thm2:concentration-new-prodnormal}, we have
\begin{align*}
\bP\left(\left|\frac{1}{n}\sum_{i=1}^n Z_iW_i\right|>\frac{1}{2}\sqrt{\frac{10\log\frac{1}{\delta}}{n}\max\left(1,\frac{10\log\frac{1}{\delta}}{n}\right)} \right)
&\leq 2 \delta;
\end{align*}
using Hoeffding's inequality, 
\begin{align*}
\bP\left( |\overline{Y}| \geq  \sqrt{\frac{2\log\frac{1}{\delta}}{n}} \right) \leq 2 \delta;
\end{align*}
and using the Gaussian tail bound, for $\delta\leq \frac{1}{\sqrt{e}}$,
\begin{align*}
\bP\left( |\overline{Z}| \geq  \sqrt{\frac{2\log\frac{1}{\delta}}{n}} \right) \leq \delta
\end{align*}
and the final result follows easily.
\end{proof} } {}

\begin{prop}\label{thm2:anglebound-new}
Let $\theta=(\mu_0-\mu,\mu_0+\mu)$ for some $\mu_0,\mu\in\bR^d$ and $X_1,...,X_n\simiid P_\theta$ with $d>1$ and $n\geq 4d$.
For any $0<\delta<\frac{d-1}{\sqrt{e}}$, if
\begin{align*}
\max\left(\frac{\sigma^2}{\|\mu\|^2},\frac{\sigma}{\|\mu\|}\right) \sqrt{\frac{\max(d,8\log\frac{1}{\delta})}{n}} \leq \frac{1}{160}
\end{align*}
then with probability at least $1-12\delta-2\exp\left(-\frac{n}{20}\right)$,
\begin{align*}
\sqrt{1-(v_1(\sigma^2I+\mu\mu^T)^Tv_1(\widehat{\Sigma}_n))^2}  
\leq 14\max\left(\frac{\sigma^2}{\|\mu\|^2}, \frac{\sigma}{\|\mu\|}\right) \sqrt{d}  \sqrt{\frac{10\log\frac{d}{\delta}}{n}\max\left(1,\frac{10\log\frac{d}{\delta}}{n}\right)} .
\end{align*}
\end{prop}
\iftoggle{detailed3}{
\begin{proof}
By Proposition \ref{thm2:concentration-new-gmmixture-spectral} (with $\delta_1=\exp\left(-\frac{n}{20}\right)$), Proposition \ref{thm2:concentration-new-offdiagonal} (with $\delta_2=\frac{\delta}{d-1}$), and Lemma \ref{thm2:DK-tighter-one}, with probability at least $1-12\delta-2\exp\left(-\frac{n}{20}\right)$,
\begin{align*}
&\sqrt{1-(v_1(\sigma^2I+\mu\mu^T)^Tv_1(\widehat{\Sigma}_n))^2}  \\
&\leq \frac{4\sqrt{d-1}}{\|\mu\|^2} \left[ \sigma^2 \frac{1}{2}\sqrt{\frac{10\log\frac{d-1}{\delta}}{n}\max\left(1,\frac{10\log\frac{d-1}{\delta}}{n}\right)} 
+ \sigma\|\mu\| \sqrt{\frac{2\log\frac{d-1}{\delta}}{n}}  + (\sigma^2+  \sigma\|\mu\|) \frac{2\log\frac{d-1}{\delta}}{n}\right]  
\end{align*}
and the result follows after some simplifications.
%
%
%
%
%
%

\end{proof}}{}

\subsection{General result relating error in estimating mean and principal direction to clustering loss}

\begin{prop}\label{thm2:upper-clusterlossbnd-general}
Let $\theta=(\mu_0-\mu,\mu_0+\mu)$ and let 
$$
\widehat{F}(x) = \left\{
\begin{array}{rl}
1 & \mbox{if } x^Tv\geq x_0^Tv \\
2 & \mbox{otherwise} 
\end{array}
\right.
$$
for some $x_0,v\in\bR^d$, with $\|v\|=1$.
Define $\cos\beta=|v^T\mu|/\|\mu\|$.
If $|(x_0-\mu_0)^Tv|\leq \sigma\epsilon_1+\|\mu\|\epsilon_2$ for some $\epsilon_1\geq 0$ and $0\leq\epsilon_2\leq\frac{1}{4}$, and if $\sin\beta\leq\frac{1}{\sqrt{5}}$, then
\begin{align*}
L_\theta(\widehat{F})
\leq \exp\left\{-\frac{1}{2}\max\left(0,\frac{\|\mu\|}{2\sigma}-2\epsilon_1\right)^2 \right\}
 \left[ 2\epsilon_1+\epsilon_2\frac{\|\mu\|}{\sigma}  
 + 2\sin\beta\left(2\sin\beta\frac{\|\mu\|}{\sigma}+ 1\right) \right] .
\end{align*}
\end{prop}

\iftoggle{detailed3} {
\begin{proof}

\begin{align*}
L_\theta(\widehat{F})
=& \min\limits_\pi P_\theta(\{x:\;F_\theta(x)\neq \pi(\widehat{F}(x))\}) \\
=& \min\left\{ P_\theta[\{x:\;((x-\mu_0)^T\mu)((x-x_0)^Tv)\geq 0\}], \right. \\
 & \;\;\;\;\;\;\;\;  \left. P_\theta[\{x:\;((x-\mu_0)^T\mu)((x-x_0)^Tv)\leq 0\}] \right\}.
\end{align*}
WLOG assume $v^T\mu\geq 0$ (otherwise we can simply replace $v$ with $-v$, which does not affect the bound).
Then
\begin{align*}
L_\theta(\widehat{F})
&= P_\theta[\{x:\;((x-\mu_0)^T\mu)((x-x_0)^Tv)\leq 0\}] \\
&= P_\theta[\{x:\;((x-\mu_0)^T\mu)((x-\mu_0)^Tv-(x_0-\mu_0)^Tv)\leq 0\}] \\
&= P_\theta\left[\left\{x:\;\left((x-\mu_0)^T\frac{\mu}{\|\mu\|}\right)((x-\mu_0)^Tv-(x_0-\mu_0)^Tv)\leq 0\right\}\right].
\end{align*}
Define 
$$
\breve{\mu} = \frac{\mu}{\|\mu\|},
$$
$$
\breve{x} = (x-\mu_0)^T \breve{\mu},
$$
and
$$
\breve{y} = (x-\mu_0)^T \frac{v-\breve{\mu}\breve{\mu}^Tv}{\|v-\breve{\mu}\breve{\mu}^Tv\|} \equiv (x-\mu_0)^T \frac{v-\breve{\mu}\breve{\mu}^Tv}{\sin\beta}
$$
so that
\begin{align*}
L_\theta(\widehat{F})
&= P_\theta\left[\left\{x:\;\breve{x}\left(\breve{y}\sin\beta +\breve{x}\cos\beta-(x_0-\mu_0)^Tv\right)\leq 0\right\}\right] \\
&= P_\theta\left[\left\{x:\; \min(0,B(\breve{y})) \leq \breve{x} \leq \max(0,B(\breve{y}))\right\}\right]
\end{align*}
where
$$
B(\breve{y}) = \frac{(x_0-\mu_0)^Tv}{\cos\beta}-\breve{y}\tan\beta.
$$
Since $\breve{x}$ and $\breve{y}$ are projections of $x-\mu_0$ onto orthogonal unit vectors, and since $\breve{x}$ is exactly the component of $x-\mu_0$ that lies in the direction of $\mu$, we can integrate out all other directions and obtain
\begin{align*}
L_\theta(\widehat{F})
&=  \int\limits_{-\infty}^\infty \phi_\sigma(\breve{y}) \int\limits_{\min(0,B(\breve{y}))}^{\max(0,B(\breve{y}))} \left(\frac{1}{2}\phi_\sigma(\breve{x}+\|\mu\|) + \frac{1}{2}\phi_\sigma(\breve{x}-\|\mu\|)\right) d\breve{x} d\breve{y} 
\end{align*}
where $\phi_\sigma$ is the density of $\cN(0,\sigma^2)$.
But,
\begin{align*}
&\int\limits_{\min(0,B(\breve{y}))}^{\max(0,B(\breve{y}))} \left(\frac{1}{2}\phi_\sigma(\breve{x}+\|\mu\|) + \frac{1}{2}\phi_\sigma(\breve{x}-\|\mu\|)\right) d\breve{x} \\
&=\frac{1}{2}\int\limits_{\min(0,B(\breve{y}))}^{\max(0,B(\breve{y}))} \phi_\sigma(\breve{x}+\|\mu\|) d\breve{x} + \frac{1}{2}\int\limits_{\min(0,B(\breve{y}))}^{\max(0,B(\breve{y}))} \phi_\sigma(\breve{x}-\|\mu\|) d\breve{x} \\
&=\frac{1}{2} \left(\Phi\left(\frac{\max(0,B(\breve{y}))+\|\mu\|}{\sigma}\right)-\Phi\left(\frac{\min(0,B(\breve{y}))+\|\mu\|}{\sigma}\right)\right) \\
&+ \frac{1}{2} \left(-\Phi\left(\frac{-\max(0,B(\breve{y}))+\|\mu\|}{\sigma}\right)+\Phi\left(\frac{-\min(0,B(\breve{y}))+\|\mu\|}{\sigma}\right)\right) \\
&=\frac{1}{2} \left(\Phi\left(\frac{\|\mu\|+|B(\breve{y})|}{\sigma}\right)-\Phi\left(\frac{\|\mu\|-|B(\breve{y})|}{\sigma}\right)\right) .
\end{align*}
Since the above quantity is increasing in $|B(\breve{y})|$, and since $|B(\breve{y})|\leq |\breve{y}|\tan\beta + r$ where
$$
r = \left|\frac{(x_0-\mu_0)^Tv}{\cos\beta}\right|,
$$
we have that, replacing $\breve{y}$ by $x$,
\begin{align*}
L_\theta(\widehat{F})
&\leq \frac{1}{2}\int\limits_{-\infty}^{\infty} \frac{1}{\sigma}\phi\left(\frac{x}{\sigma}\right)
\left[\Phi\left(\frac{\|\mu\|+|x|\tan\beta + r}{\sigma}\right) -\Phi\left(\frac{\|\mu\|-|x|\tan\beta - r}{\sigma}\right) \right]dx \\
&\leq \int\limits_{-\infty}^{\infty} \frac{1}{\sigma}\phi\left(\frac{x}{\sigma}\right)
\left[\Phi\left(\frac{\|\mu\|}{\sigma}\right) -\Phi\left(\frac{\|\mu\|-|x|\tan\beta - r}{\sigma}\right) \right]dx \\
&= \int\limits_{-\infty}^{\infty} \phi(x)
\left[\Phi\left(\frac{\|\mu\|-r}{\sigma}\right) -\Phi\left(\frac{\|\mu\|- r}{\sigma}-|x|\tan\beta\right) \right]dx \\
&+ \left[\Phi\left(\frac{\|\mu\|}{\sigma}\right) -\Phi\left(\frac{\|\mu\| - r}{\sigma}\right) \right] .
\end{align*}
Since $\tan\beta\leq \frac{1}{2}$, we have that $r\leq2|(x_0-\mu_0)^Tv|\leq 2\sigma\epsilon_1+2\|\mu\|\epsilon_2$ and
\begin{align*}
\Phi\left(\frac{\|\mu\|}{\sigma}\right) -\Phi\left(\frac{\|\mu\| - r}{\sigma}\right) 
&\leq \frac{r}{\sigma} \phi\left(\max\left(0,\frac{\|\mu\| - r}{\sigma}\right)\right)   \\
&\leq  \left(2\epsilon_1+2\epsilon_2\frac{\|\mu\|}{\sigma} \right)\phi\left(\max\left(0,(1-2\epsilon_2)\frac{\|\mu\|}{\sigma} - 2\epsilon_1\right)\right)   ,
\end{align*}
and since $\epsilon_2\leq\frac{1}{4}$,
\begin{align*}
\Phi\left(\frac{\|\mu\|}{\sigma}\right) -\Phi\left(\frac{\|\mu\| - r}{\sigma}\right) 
&\leq  2\left(\epsilon_1+\epsilon_2\frac{\|\mu\|}{\sigma} \right)\phi\left(\max\left(0,\frac{\|\mu\|}{2\sigma} - 2\epsilon_1\right)\right).
\end{align*}

Defining $A=\left|\frac{\|\mu\|- r}{\sigma}\right|$, 
\begin{align*}
&\int\limits_{-\infty}^{\infty} \phi(x)
\left[\Phi\left(\frac{\|\mu\|-r}{\sigma}\right) -\Phi\left(\frac{\|\mu\|- r}{\sigma}-|x|\tan\beta\right) \right]dx  \\
&\leq 
2\int\limits_{0}^{\infty} \int\limits_{A-x\tan\beta}^{A}  \phi(x) \phi(y) dydx
=2\int\limits_{-A\sin\beta}^{\infty} \int\limits_{A\cos\beta}^{A\cos\beta+(x+A\sin\beta)\tan\beta}  \phi(x) \phi(y) dydx \\
&\leq 2\phi(A\cos\beta)\tan\beta \int\limits_{-A\sin\beta}^{\infty} (x+A\sin\beta)  \phi(x) dx \\
&= 2\phi(A\cos\beta)\tan\beta \left(A\sin\beta\Phi(A\sin\beta) +  \phi(A\sin\beta)\right) \\
&\leq 2\phi\left(A\right)\tan\beta \left(A\sin\beta + 1\right) \\
&\leq 2\phi\left(\max\left(0,\frac{\|\mu\|- r}{\sigma}\right)\right)\tan\beta \left(\left(\frac{\|\mu\|+ r}{\sigma}\right)\sin\beta + 1\right)
\end{align*}
and
\begin{align*}
&\int\limits_{-\infty}^{\infty} \phi(x)
\left[\Phi\left(\frac{\|\mu\|-r}{\sigma}\right) -\Phi\left(\frac{\|\mu\|- r}{\sigma}-|x|\tan\beta\right) \right]dx  \\
&\leq 2\phi\left(\max\left(0,\frac{\|\mu\|}{2\sigma}-2\epsilon_1\right)\right)\tan\beta \left(\left(2\frac{\|\mu\|}{\sigma}+2\epsilon_1\right)\sin\beta + 1\right).
\end{align*}
So we have that
\begin{align*}
L_\theta(\widehat{F})
&\leq  2\left(\epsilon_1+\epsilon_2\frac{\|\mu\|}{\sigma} \right)\phi\left(\max\left(0,\frac{\|\mu\|}{2\sigma} - 2\epsilon_1\right)\right) \\
&+ 2\phi\left(\max\left(0,\frac{\|\mu\|}{2\sigma}-2\epsilon_1\right)\right)\tan\beta \left(\left(2\frac{\|\mu\|}{\sigma}+2\epsilon_1\right)\sin\beta + 1\right) \\
&\leq \phi\left(\max\left(0,\frac{\|\mu\|}{2\sigma}-2\epsilon_1\right)\right) \times \\
& \times \left[ 2\epsilon_1+2\epsilon_2\frac{\|\mu\|}{\sigma}  
 + 4\sin\beta\tan\beta\frac{\|\mu\|}{\sigma}+4\epsilon_1\sin\beta\tan\beta + 2\tan\beta \right]  \\
&\leq \exp\left\{-\frac{1}{2}\max\left(0,\frac{\|\mu\|}{2\sigma}-2\epsilon_1\right)^2 \right\}
 \left[ 2\epsilon_1+\epsilon_2\frac{\|\mu\|}{\sigma}  
 + \tan\beta\left(2\sin\beta\frac{\|\mu\|}{\sigma}+ 1\right) \right] .
\end{align*}

\end{proof}}{}

\subsection{Non-sparse upper bound}

\begin{theorem}\label{thm2:u5-ns-upper}
For any $\theta\in\Theta_\lambda$ and $X_1,...,X_n\simiid P_\theta$, let
$$
\widehat{F}(x) = \left\{
\begin{array}{rl}
1 & \mbox{if } x^Tv_1(\widehat{\Sigma}_n)\geq \widehat{\mu}_n^Tv_1(\widehat{\Sigma}_n) \\
2 & \mbox{otherwise,} 
\end{array}
\right.
$$
and let $n\geq \max(68,4d)$, $d\geq1$.

Then
\begin{align*}
\sup_{\theta\in\Theta_\lambda} \bE L_\theta(\widehat{F})
\leq 600 \max\left(\frac{4\sigma^2}{\lambda^2}, 1\right)   \sqrt{\frac{d\log(nd)}{n}}.
\end{align*}

Furthermore, if $\frac{\lambda}{\sigma} \geq 2\max(80,14\sqrt{5d})$, then
\begin{align*}
\sup_{\theta\in\Theta_\lambda} \bE L_\theta(\widehat{F})
\leq 17\exp\left(-\frac{n}{32}\right)
+ 9\exp\left(-\frac{\lambda^2}{80\sigma^2} \right).
\end{align*}
\end{theorem}

\begin{proof}

Using Propositions \ref{thm2:upper-new-meandiff} and \ref{thm2:anglebound-new} with $\delta=\frac{1}{\sqrt{n}}$, Proposition \ref{thm2:upper-clusterlossbnd-general}, and the fact that  $(C+x)\exp(-\max(0,x-4)^2/8)\leq(C+6)\exp(-\max(0,x-4)^2/10)$ for all $C,x>0$,
\begin{align*}
\bE L_\theta(\widehat{F})
&\leq 600 \max\left(\frac{4\sigma^2}{\lambda^2}, 1\right)   \sqrt{\frac{d\log(nd)}{n}}
\end{align*}
(it is easy to verify that the bounds are decreasing with $\|\mu\|$, so we use $\|\mu\|=\frac{\lambda}{2}$ to bound the supremum).
Note that the $d=1$ case must be handled separately, but results in a bound that agrees with the above.

Also, when $\frac{\lambda}{\sigma} \geq 2\max(80,14\sqrt{5d})$, using $\delta=\exp\left(-\frac{n}{32}\right)$,
\begin{align*}
\bE L_\theta(\widehat{F})
\leq& 17\exp\left(-\frac{n}{32}\right)
+ 9\exp\left(-\frac{\lambda^2}{80\sigma^2} \right).
\end{align*}
\end{proof}

\subsection{Estimating the support in the sparse case}

\begin{prop}\label{thm2:upperbound-new-variance-perdim}
Let $\theta=(\mu_0-\mu,\mu_0+\mu)$ for some $\mu_0,\mu\in\bR^d$ and $X_1,...,X_n\simiid P_\theta$.
For any $0<\delta<\frac{1}{\sqrt{e}}$ such that $\sqrt{\frac{6\log \frac{1}{\delta} }{n}}\leq \frac{1}{2}$, with probability at least $1-6d\delta$,
$$
|\widehat{\Sigma}_n(i,i) - (\sigma^2+\mu(i)^2)| \leq \sigma^2 \sqrt{\frac{6\log \frac{1}{\delta} }{n}} + 2\sigma|\mu(i)|\sqrt{\frac{2\log\frac{1}{\delta}}{n}} + (\sigma+|\mu(i)|)^2\frac{2\log\frac{1}{\delta}}{n}
$$
for all $i\in[d]$.
\end{prop}
\iftoggle{detailed3}{
\begin{proof}
Consider any $i\in[d]$.
Let $Z_1,...,Z_n\simiid \cN(0,1)$ and $Y_1,...,Y_n$ i.i.d. such that $\bP(Y_j=-1)=\bP(Y_j=1)=\frac{1}{2}$.
Then $\widehat{\Sigma}_n(i,i)$ is equal in distribution to
$$
\frac{1}{n}\sum_{j=1}^n (\sigma Z_j + \mu(i) Y_j -\sigma \overline{Z} - \mu(i) \overline{Y})^2
$$
where $ \overline{Z}$ and $ \overline{Y}$ are the respective empirical means, and
\begin{align*}
\frac{1}{n}\sum_{j=1}^n (\sigma Z_j+ \mu(i) Y_j -\sigma \overline{Z} - \mu(i) \overline{Y})^2 
&= \frac{1}{n}\sum_{j=1}^n (\sigma Z_j + \mu(i) Y_j)^2 - (\sigma \overline{Z} + \mu(i) \overline{Y})^2  \\
&= \sigma^2 \frac{1}{n}\sum_{j=1}^n  Z_j^2 +  \mu(i)^2 + 2 \sigma\mu(i) \frac{1}{n}\sum_{j=1}^n  Z_j Y_j \\
&- \sigma^2 \overline{Z}^2 -  \mu(i)^2 \overline{Y}^2 - 2 \sigma \mu(i) \overline{Z}  \overline{Y}
\end{align*}
So, by Hoeffding's inequality, a Gaussian tail bound, and Proposition \ref{thm2:concentration-new-chisq}, we have that for any $0<\delta<\frac{1}{\sqrt{e}}$, with probability at least $1-6\delta$,
$$
|\widehat{\Sigma}_n(i,i) - (\sigma^2+\mu(i)^2)| \leq \sigma^2 \sqrt{\frac{6\log \frac{1}{\delta} }{n}} + 2\sigma|\mu(i)|\sqrt{\frac{2\log\frac{1}{\delta}}{n}} + (\sigma+|\mu(i)|)^2\frac{2\log\frac{1}{\delta}}{n}
$$
where we have used the fact that for $\epsilon\in(0,0.5]$, 
$$
 \max\left\{-\epsilon + \log(1+\epsilon),\; \epsilon + \log (1-\epsilon)\right\}\leq -\frac{\epsilon^2}{3}
$$
and the result follows easily.
\end{proof} }{}

\begin{prop}\label{thm2:upperbound-new-support-recovery}
Let $\theta=(\mu_0-\mu,\mu_0+\mu)$ for some $\mu_0,\mu\in\bR^d$ and $X_1,...,X_n\simiid P_\theta$.
Define
$$
S(\theta)=\{i\in[d]: \mu(i)\neq 0\},
$$
$$
\alpha = \sqrt{\frac{6\log (nd) }{n}}  + \frac{2\log(nd)}{n},
$$
$$
\widetilde{S}(\theta)=\{i\in[d]: |\mu(i)|\geq 4\sigma\sqrt{\alpha}\},
$$
$$
\widehat{\tau}_n = \frac{1+\alpha}{1-\alpha} \min_{i\in[d]} \widehat{\Sigma}_n(i,i),
$$
and
$$
\widehat{S}_n=\{i\in[d]: \widehat{\Sigma}_n(i,i) > \widehat{\tau}_n \}.
$$
Assume that $n\geq 1$, $d\geq2$, and $\alpha \leq \frac{1}{4}$.
Then $\widetilde{S}(\theta) \subseteq \widehat{S}_n \subseteq S(\theta)$ with probability at least $1-\frac{6}{n}$.
\end{prop}
\iftoggle{detailed3} {
\begin{proof}
By Proposition \ref{thm2:upperbound-new-variance-perdim}, with probability at least $1-\frac{6}{n}$,
$$
|\widehat{\Sigma}_n(i,i) - (\sigma^2+\mu(i)^2)| \leq \sigma^2 \sqrt{\frac{6\log (nd) }{n}} + 2\sigma|\mu(i)|\sqrt{\frac{2\log(nd)}{n}} + (\sigma+|\mu(i)|)^2\frac{2\log(nd)}{n}
$$
for all $i\in[d]$.
Assume the above event holds.
If $S(\theta)=[d]$, then of course $\widehat{S}_n \subseteq S(\theta)$.
Otherwise, for $i\notin S(\theta)$,
$$
(1-\alpha)\sigma^2 \leq \widehat{\Sigma}_n(i,i)  \leq (1+\alpha)\sigma^2
$$
so it is clear that $\widehat{S}_n \subseteq S(\theta)$.

The remainder of the proof is trivial if $\widetilde{S}(\theta) =\emptyset$ or $S(\theta)=\emptyset$.
Assume otherwise.
For any $i \in S(\theta)$,
\begin{align*}
\widehat{\Sigma}_n(i,i) 
&\geq (1-\alpha)\sigma^2 +\mu(i)^2  - 2\sigma|\mu(i)|\sqrt{\frac{2\log(nd)}{n}} - 2\sigma|\mu(i)|\frac{2\log(nd)}{n} - \mu(i)^2\frac{2\log(nd)}{n}    \\
&\geq (1-\alpha)\sigma^2 + \left(1-\frac{2\log(nd)}{n}\right) \mu(i)^2  - 2\alpha \sigma|\mu(i)|.
\end{align*}
By definition, $|\mu(i)|\geq 4\sigma\sqrt{\alpha}$ for all $i \in \widetilde{S}(\theta)$, so
$$
\frac{(1+\alpha)^2}{1-\alpha} \sigma^2\leq (1-\alpha)\sigma^2 + \left(1-\frac{2\log(nd)}{n}\right) \mu(i)^2  - 2\alpha \sigma|\mu(i)|\leq \widehat{\Sigma}_n(i,i) 
$$
and $i \in \widehat{S}_n$ (we ignore strict equality above as a measure $0$ event), i.e. $\widetilde{S}(\theta)\subseteq \widehat{S}_n$, which concludes the proof.
\end{proof} }{}

\subsection{Sparse upper bound}

\begin{theorem}\label{thm2:upperbound-new-support-recovery-plus-estimation}
For any $\theta=(\mu_0-\mu,\mu_0+\mu)\in\Theta_{\lambda,s}$ and $X_1,...,X_n\simiid P_\theta$ with $n\geq \max(68,4s)$ and $s\geq1$, define
$$
\alpha = \sqrt{\frac{6\log (nd) }{n}}  + \frac{2\log(nd)}{n},
$$
$$
\widehat{\tau}_n = \frac{1+\alpha}{1-\alpha} \min_{i\in[d]} \widehat{\Sigma}_n(i,i),
$$
and
$$
\widehat{S}_n=\{i\in[d]: \widehat{\Sigma}_n(i,i) > \widehat{\tau}_n \}.
$$
Assume that $d\geq2$, and $\alpha \leq \frac{1}{4}$.
Let
$$
\widehat{F}_n(x) = \left\{
\begin{array}{rl}
1 & \mbox{if } x_{\widehat{S}_n}^Tv_1(\widehat{\Sigma}_{\widehat{S}_n})\geq \widehat{\mu}_{\widehat{S}_n}^Tv_1(\widehat{\Sigma}_{\widehat{S}_n}) \\
2 & \mbox{otherwise} 
\end{array}
\right.
$$
where $\widehat{\mu}_{\widehat{S}_n}$ and $\widehat{\Sigma}_{\widehat{S}_n}$ are the empirical mean and covariance of $X_i$ for the dimensions in $\widehat{S}_n$, and $0$ elsewhere.
Then
\begin{align*}
\sup_{\theta\in\Theta_{\lambda,s}}\bE L_\theta(\widehat{F})
\leq 603 \max\left(\frac{16\sigma^2}{\lambda^2}, 1\right)   \sqrt{\frac{s\log(ns)}{n}} + 220\frac{\sigma\sqrt{s}}{\lambda} \left(\frac{\log (nd) }{n}\right)^{\frac{1}{4}}.
\end{align*}
\end{theorem}

\begin{proof}
Define
$$
S(\theta)=\{i\in[d]: \mu(i)\neq 0\}
$$
and
$$
\widetilde{S}(\theta)=\{i\in[d]: |\mu(i)|\geq 4\sigma\sqrt{\alpha}\},
$$
Assume $\widetilde{S}(\theta) \subseteq \widehat{S}_n \subseteq S(\theta)$ (by Proposition \ref{thm2:upperbound-new-support-recovery}, this holds with probability at least $1-\frac{6}{n}$).
If $\widetilde{S}(\theta)=\emptyset$, then we simply have $\bE L_\theta(\widehat{F}_n) \leq \frac{1}{2}$.

Assume $\widetilde{S}(\theta)\neq\emptyset$.
Let 
$$\cos\widehat{\beta} = |v_1(\widehat{\Sigma}_{\widehat{S}_n})^Tv_1(\Sigma)|,$$
$$\cos\widetilde{\beta} = |v_1(\Sigma_{\widehat{S}_n})^Tv_1(\Sigma)|,$$
and
$$\cos\beta = |v_1(\widehat{\Sigma}_{\widehat{S}_n})^Tv_1(\Sigma_{\widehat{S}_n})|$$
where $\Sigma=\sigma^2I+\mu\mu^T$, and $\Sigma_{\widehat{S}_n}$ is the same as $\Sigma$ in $\widehat{S}_n$, and $0$ elsewhere.
Then
$$\sin\widehat{\beta}\leq \sin\widetilde{\beta} + \sin\beta.$$
Also
\begin{align*}
\sin\widetilde{\beta} &= \frac{\|\mu-\mu_{\widehat{S}(\theta)}\|}{\|\mu\|} \\
&\leq \frac{\|\mu-\mu_{\widetilde{S}(\theta)}\|}{\|\mu\|}  \\
&\leq \frac{4\sigma\sqrt{\alpha}\sqrt{|S(\theta)|-|\widetilde{S}(\theta)}|}{\|\mu\|}  \\
&\leq 8\frac{\sigma\sqrt{s\alpha}}{\lambda}.
\end{align*}

Using the same argument as the proof of Theorem \ref{thm2:u5-ns-upper}, we have that as long as the above bound is smaller than $\frac{1}{2\sqrt{5}}$,
\begin{align*}
\bE L_\theta(\widehat{F})
&\leq 600 \max\left(\frac{\sigma^2}{\left(\frac{\lambda}{2}-4\sigma\sqrt{s\alpha}\right)^2}, 1\right)   \sqrt{\frac{s\log(ns)}{n}} + 104\frac{\sigma\sqrt{s\alpha}}{\lambda} + \frac{3}{n} \\
&\leq 603 \max\left(16\frac{\sigma^2}{\lambda^2}, 1\right)   \sqrt{\frac{s\log(ns)}{n}} + 104\frac{\sigma\sqrt{s\alpha}}{\lambda} .
\end{align*}
However, when $8\frac{\sigma\sqrt{s\alpha}}{\lambda}>\frac{1}{2\sqrt{5}}$, the above bound is bigger than $\frac{1}{2}$, which is a trivial upper bound on the clustering error, hence the bound can be stated without further conditions.
Finally, since $\alpha\leq\frac{1}{4}$, we must have $\frac{\log(nd)}{n}\leq 1$, so $\alpha\leq (\sqrt{6}+2)\sqrt{\frac{\log (nd) }{n}}$, which completes the proof.

%
%

\end{proof}

\section{Lower bounds}

\subsection{Standard tools}

\begin{lemma}\label{thm2:fanobasic}
Let $P_0,P_1,...,P_M$ be probability measures satisfying
$$
\frac{1}{M}\sum_{i=1}^M \KL(P_i,P_0)\leq \alpha \log M
$$
where $0<\alpha<1/8$ and $M\geq2$.
Then
$$
\inf_{\psi} \max_{i\in[0..M]} P_i(\psi\neq i) \geq 0.07
$$
(\cite{tsybakov2009}).
\end{lemma}

\begin{lemma}\label{thm2:varshamov} \textbf{(Varshamov--Gilbert bound)}
Let $\Omega=\{0,1\}^m$ for $m\geq8$.
Then there exists a subset $\{\omega_0,...,\omega_M\}\subseteq\Omega$ such that $\omega_0=(0,...,0)$,
$$
\rho(\omega_i,\omega_j)\geq \frac{m}{8}, \;\;\forall\; 0\leq i<j\leq M,
$$
and
$$
M\geq 2^{m/8},
$$
where $\rho$ denotes the Hamming distance between two vectors (\cite{tsybakov2009}).
\end{lemma}

\begin{lemma}\label{thm2:varshamovsparse}
Let $\Omega=\{\omega\in\{0,1\}^m:\|\omega\|_0=s\}$ for integers $m>s\geq1$.
For any $\alpha,\beta\in(0,1)$ such that $s\leq\alpha\beta m$, there exists $\omega_0,...,\omega_M\in\Omega$ such that for all $0\leq i<j\leq M$,
$$
  \rho(\omega_i,\omega_j)> 2(1-\alpha) s
$$
  and
$$
  \log (M+1)\geq c s \log\left(\frac{m}{s}\right)
$$
where
$$
c=\frac{\alpha}{-\log(\alpha\beta)}(-\log\beta+\beta-1).
$$
In particular, setting $\alpha=3/4$ and $\beta=1/3$, we have that $\rho(\omega_i,\omega_j)>s/2$, $\log(M+1)\geq \frac{s}{5}  \log\left(\frac{m}{s}\right)$, as long as $s\leq m/4$ (\cite{massart2007}, Lemma 4.10).
\end{lemma}

\subsection{A reduction to hypothesis testing without a general triangle inequality}

\begin{prop}\label{thm2:lowerbnd-gen}
Let $\theta_0,...,\theta_M\in\Theta_{\lambda}$ (or $\Theta_{\lambda,s}$), $M\geq 2$, $0<\alpha<1/8$, and $\gamma>0$.
If
$$
\max_{i\in[M]} \KL(P_{\theta_i},P_{\theta_0})\leq \frac{\alpha \log M}{n}
$$
and for all $0\leq i\neq j \leq M$ and clusterings $\widehat{F}$,
$$
L_{\theta_i}(\widehat{F})<\gamma \mbox{ implies } L_{\theta_j}(\widehat{F}) \geq \gamma,
$$
then
\begin{align*}
\inf_{\widehat{F}_n}\max_{i\in[0..M]} \bE_{\theta_i} L_{\theta_i}(\widehat{F}_n) &\geq 0.07\gamma.
\end{align*}
\end{prop}
\iftoggle{detailed2}{
\begin{proof}
Using Markov's inequality,
\begin{align*}
\inf_{\widehat{F}_n}\max_{i\in[0..M]} \bE_{\theta_i} L_{\theta_i}(\widehat{F}_n) &\geq \gamma \inf_{\widehat{F}_n}\max_{i\in[0..M]} P_{\theta_i}^n \left(L_{\theta_i}(\widehat{F}_n)\geq \gamma\right).
\end{align*}
Define $\psi^*(\widehat{F}_n)=\argmin\limits_{i\in[0..M]} L_{\theta_i}(\widehat{F}_n)$.
By assumption, $L_{\theta_i}(\widehat{F}_n)<\gamma$ implies $L_{\theta_j}(\widehat{F}_n)\geq \gamma$ for any $j\neq i$, so $L_{\theta_i}(\widehat{F}_n)<\gamma$ only when $\psi^*(\widehat{F}_n)=i$.
Hence,
$$
P_{\theta_i}^n\left(\psi^*(\widehat{F}_n)=i\right)\geq P_{\theta_i}^n \left( L_{\theta_i}(\widehat{F}_n) < \gamma \right)
$$
and
\begin{align*}
\inf_{\widehat{F}_n}\max_{i\in[0..M]} P_{\theta_i}^n \left(L_{\theta_i}(\widehat{F}_n)\geq \gamma\right)
&\geq \max_{i\in[0..M]} P_{\theta_i}^n\left(\psi^*(\widehat{F}_n)\neq i\right) \\
&\geq \inf_{\widehat{\psi}_n}\max_{i\in[0..M]} P_{\theta_i}^n\left(\widehat{\psi}_n\neq i\right) \\
&\geq 0.07
\end{align*}
where the last step is by Lemma \ref{thm2:fanobasic}.
\end{proof}
}{}

\subsection{Properties of the clustering error}

\begin{prop}\label{thm2:clusterloss-triangle}
For any $\theta,\theta'\in\Theta_{\lambda}$, and any clustering $\widehat{F}$, if
$$
L_\theta(F_{\theta'})+  L_\theta(\widehat{F})+\sqrt{\frac{\KL(P_\theta,P_{\theta'})}{2}}\leq \frac{1}{2},
$$
then
\begin{align*}
L_\theta(F_{\theta'})-L_\theta(\widehat{F}) -\sqrt{\frac{\KL(P_\theta,P_{\theta'})}{2}}
\leq L_{\theta'}(\widehat{F})
\leq L_\theta(F_{\theta'})+  L_\theta(\widehat{F})+\sqrt{\frac{\KL(P_\theta,P_{\theta'})}{2}}.
\end{align*}
\end{prop}
\iftoggle{detailed2}{
\begin{proof}
WLOG assume $F_\theta$, $F_{\theta'}$, and $\widehat{F}$ are such that, using simplified notation,
$$
L_\theta(F_{\theta'}) = P_\theta(F_\theta\neq F_{\theta'})
$$
and
$$
L_\theta(\widehat{F}) = P_\theta(F_\theta\neq \widehat{F}).
$$
Then
\begin{align*}
P_\theta(F_{\theta'}\neq\widehat{F})
&= P_\theta\left((F_\theta=F_{\theta'})\cap(F_\theta\neq\widehat{F})\cup(F_\theta\neq F_{\theta'})\cap(F_\theta=\widehat{F})\right) \\
&= P_\theta\left((F_\theta=F_{\theta'})\cap(F_\theta\neq\widehat{F})\right)+P_\theta\left((F_\theta\neq F_{\theta'})\cap(F_\theta=\widehat{F})\right).
\end{align*}
Since
\begin{align*}
0\leq P_\theta\left((F_\theta=F_{\theta'})\cap(F_\theta\neq\widehat{F})\right) \leq P_\theta\left(F_\theta\neq\widehat{F}\right)=L_\theta(\widehat{F}),
\end{align*}
\begin{align*}
P_\theta\left((F_\theta\neq F_{\theta'})\cap(F_\theta=\widehat{F})\right) \leq P_\theta\left(F_\theta\neq F_{\theta'}\right) = L_\theta(F_{\theta'}),
\end{align*}
and
\begin{align*}
L_\theta(F_{\theta'})-L_\theta(\widehat{F}) = P_\theta\left(F_\theta\neq F_{\theta'}\right)-P_\theta(F_\theta\neq\widehat{F})\leq P_\theta\left((F_\theta\neq F_{\theta'})\cap(F_\theta=\widehat{F})\right),
\end{align*}
we have that 
\begin{align*}
L_\theta(F_{\theta'})-L_\theta(\widehat{F})
\leq P_\theta(F_{\theta'}\neq\widehat{F})
\leq L_\theta(F_{\theta'})+  L_\theta(\widehat{F})
\end{align*}
and
\begin{align*}
L_\theta(F_{\theta'})-L_\theta(\widehat{F}) -\TV(P_\theta,P_{\theta'})
\leq P_{\theta'}(F_{\theta'}\neq\widehat{F})
\leq L_\theta(F_{\theta'})+  L_\theta(\widehat{F})+\TV(P_\theta,P_{\theta'}).
\end{align*}
It is easy to see that if $L_\theta(F_{\theta'})+  L_\theta(\widehat{F})+\TV(P_\theta,P_{\theta'})\leq \frac{1}{2}$, then the above bound implies
\begin{align*}
L_\theta(F_{\theta'})-L_\theta(\widehat{F}) -\TV(P_\theta,P_{\theta'})
\leq L_{\theta'}(\widehat{F})
\leq L_\theta(F_{\theta'})+  L_\theta(\widehat{F})+\TV(P_\theta,P_{\theta'}).
\end{align*}
The final step is to use the fact that $\TV(P_\theta,P_{\theta'})\leq\sqrt{\frac{\KL(P_\theta,P_{\theta'})}{2}}$.
\end{proof}}{}

\begin{prop}\label{thm2:clusterloss-bnd-new}
For some $\mu_0,\mu,\mu'\in\bR^d$ such that $\|\mu\|=\|\mu'\|$, let
$$
\theta=\left(\mu_0-\frac{\mu}{2},\mu_0+\frac{\mu}{2}\right)
$$
and
$$
\theta'=\left(\mu_0-\frac{\mu'}{2},\mu_0+\frac{\mu'}{2}\right).
$$
Then
\begin{align*}
2g\left(\frac{\|\mu\|}{2\sigma}\right) \sin\beta \cos\beta \leq
L_\theta(F_{\theta'})
\leq \frac{1}{\pi}\tan\beta
\end{align*}
where $\cos\beta=\frac{|\mu^T\mu'|}{\|\mu\|^2}$ and $g(x)=\phi(x)(\phi(x)-x\Phi(-x))$.
\end{prop}
\iftoggle{detailed2}{
\begin{proof}
It is easy to see that
$$
L_\theta(F_{\theta'})= \frac{1}{2} \int\limits_\bR \frac{1}{\sigma}\phi\left(\frac{x}{\sigma}\right) \left(\Phi\left(\frac{\|\mu\|}{2\sigma}+\frac{|x|\tan\beta}{\sigma}\right)-\Phi\left(\frac{\|\mu\|}{2\sigma}-\frac{|x|\tan\beta}{\sigma}\right)\right) dx.
$$
Define $\xi=\frac{\|\mu\|}{2\sigma}$.
With a change of variables, we have
\begin{align*}
L_\theta(F_{\theta'})
&= \frac{1}{2} \int\limits_\bR \phi(x) \left(\Phi\left(\xi+|x|\tan\beta\right)-\Phi\left(\xi-|x|\tan\beta\right)\right) dx \\
&= \int\limits_{0}^{\infty} \phi(x) (\Phi(\xi+x\tan\beta)-\Phi(\xi-x\tan\beta)) dx.
\end{align*}
For any $a\leq b$, $\Phi(b)-\Phi(a)\leq\frac{b-a}{\sqrt{2\pi}}$, so
\begin{align*}
L_\theta(F_{\theta'})
&= \int\limits_{0}^{\infty} \phi(x) (\Phi(\xi+x\tan\beta)-\Phi(\xi-x\tan\beta)) dx \\
&\leq \int\limits_{0}^{\infty} \phi(x) (\Phi(x\tan\beta)-\Phi(-x\tan\beta)) dx \\
&\leq \tan\beta\sqrt{\frac{2}{\pi}}\int\limits_{0}^{\infty} x\phi(x) dx \\
&= \frac{1}{\pi}\tan\beta.
\end{align*}
Also,
\begin{align*}
L_\theta(F_{\theta'})
&= \int\limits_{0}^{\infty} \phi(x) (\Phi(\xi+x\tan\beta)-\Phi(\xi-x\tan\beta)) dx \\
&\geq 2\tan\beta \int\limits_{0}^{\infty} x\phi(x) \phi(\xi+x\tan\beta) dx \\
&= 2\tan\beta\frac{1}{\sqrt{2\pi}} \int\limits_{0}^{\infty} x \frac{1}{\sqrt{2\pi}} \exp\left\{-\frac{x^2+(\xi+x\tan\beta)^2}{2}\right\} dx \\
&= 2\tan\beta\frac{1}{\sqrt{2\pi}} \exp\left\{-\frac{\xi^2}{2}\left(1-\frac{\tan^2\beta}{1+\tan^2\beta}\right)\right\} \int\limits_{0}^{\infty} x \frac{1}{\sqrt{2\pi}} \exp\left\{-\frac{\left(x+\frac{\xi\tan\beta}{1+\tan^2\beta}\right)^2}{2\left(\frac{1}{\sqrt{1+\tan^2\beta}}\right)^2}\right\} dx \\
&\geq 2\tan\beta\frac{1}{\sqrt{2\pi}} \exp\left\{-\frac{\xi^2}{2}\right\} \int\limits_{0}^{\infty} x \frac{1}{\sqrt{2\pi}} \exp\left\{-\frac{\left(x+\xi\sin\beta\cos\beta\right)^2}{2\cos^2\beta}\right\} dx \\
&= 2\tan\beta \phi(\xi) \left[\frac{\cos^2\beta}{\sqrt{2\pi}}\exp\left\{-\frac{\xi^2\sin^2\beta}{2}\right\}-\xi\sin\beta\cos^2\beta \Phi(-\xi\sin\beta)\right] \\
&= 2\sin\beta \cos\beta \phi(\xi)\left[\phi(\xi\sin\beta)-\xi\sin\beta \Phi(-\xi\sin\beta)\right] \\
&\geq 2\sin\beta \cos\beta \phi(\xi) \left[\phi(\xi)-\xi \Phi(-\xi)\right].
\end{align*}
\end{proof}}{}

\subsection{A KL divergence bound of the necessary order}

\begin{prop}\label{thm2:KL-symmetric}
For some $\mu_0,\mu,\mu'\in\bR^d$ such that $\|\mu\|=\|\mu'\|$, let
$$
\theta=\left(\mu_0-\frac{\mu}{2},\mu_0+\frac{\mu}{2}\right)
$$
and
$$
\theta'=\left(\mu_0-\frac{\mu'}{2},\mu_0+\frac{\mu'}{2}\right).
$$
Then
\begin{align*}
\KL(P_\theta,P_{\theta'}) &\leq \xi^4 (1-\cos\beta)
\end{align*}
where $\xi=\frac{\|\mu\|}{2\sigma}$ and $\cos\beta=\frac{|\mu^T\mu'|}{\|\mu\|\|\mu'\|}$.
\end{prop}
\iftoggle{detailed2}{
\begin{proof}
Since the KL divergence is invariant to affine transformations, it is easy to see that 
$$
\KL(P_\theta,P_{\theta'}) = \int\limits_\bR \int\limits_\bR p_1(x,y)\log\frac{p_1(x,y)}{p_2(x,y)} dx dy
$$
where
$$
p_1(x,y)= \frac{1}{2}\phi(x+\xi_x)\phi(y+\xi_y) + \frac{1}{2}\phi(x-\xi_x)\phi(y-\xi_y),
$$
$$
p_2(x,y)= \frac{1}{2}\phi(x+\xi_x)\phi(y-\xi_y) + \frac{1}{2}\phi(x-\xi_x)\phi(y+\xi_y),
$$
$$
\xi_x=\xi\cos\frac{\beta}{2},\;\;\; \xi_y=\xi\sin\frac{\beta}{2}.
$$
Since
\begin{align*}
\frac{p_1(x,y)}{p_2(x,y)}
&= \frac{\phi(x+\xi_x)\phi(y+\xi_y) + \phi(x-\xi_x)\phi(y-\xi_y)}{\phi(x+\xi_x)\phi(y-\xi_y) + \phi(x-\xi_x)\phi(y+\xi_y)} \\
&= \frac{\exp(-x\xi_x -y\xi_y) + \exp(x\xi_x +y\xi_y)}
        {\exp(-x\xi_x +y\xi_y) + \exp(x\xi_x -y\xi_y)}
\end{align*}
we have
\begin{align*}
\log\frac{p_1(x,y)}{p_2(x,y)} 
&= \log \frac{\cosh (x\xi_x +y\xi_y)}{ \cosh (x\xi_x -y\xi_y)}.
\end{align*}
Furthermore,
\begin{align*}
&\int\limits_\bR \int\limits_\bR \frac{1}{2} \phi(x+\xi_x)\phi(y+\xi_y)  \log \frac{\cosh (x\xi_x +y\xi_y)}{ \cosh (x\xi_x -y\xi_y)} dx dy  \\
&= \int\limits_\bR \int\limits_\bR \frac{1}{2} \phi(-x+\xi_x)\phi(-y+\xi_y)  \log \frac{\cosh (-x\xi_x -y\xi_y)}{ \cosh (-x\xi_x +y\xi_y)} dx dy \\
&= \int\limits_\bR \int\limits_\bR \frac{1}{2} \phi(x-\xi_x)\phi(y-\xi_y)  \log \frac{\cosh (x\xi_x +y\xi_y)}{ \cosh (x\xi_x -y\xi_y)} dx dy 
\end{align*}
so
\begin{align*}
\KL(P_\theta,P_{\theta'}) 
&= \int\limits_\bR \int\limits_\bR \phi(x-\xi_x)\phi(y-\xi_y)  \log \frac{\cosh (x\xi_x +y\xi_y)}{ \cosh (x\xi_x -y\xi_y)} dx dy \\
&= \int\limits_\bR \int\limits_\bR \phi(x)\phi(y)  \log \frac{\cosh (x\xi_x+\xi_x^2 +y\xi_y+\xi_y^2)}{ \cosh (x\xi_x + \xi_x^2 -y\xi_y-\xi_y^2)} dx dy .
\end{align*}
But for any $x$
\begin{align*}
&- \int\limits_\bR \phi(x)\phi(y)  \log  \cosh (x\xi_x + \xi_x^2 -y\xi_y-\xi_y^2) dy  \\
&= - \int\limits_\bR \phi(x)\phi(-y)  \log  \cosh (x\xi_x + \xi_x^2 +y\xi_y-\xi_y^2) dy  \\
&= - \int\limits_\bR \phi(x)\phi(y)  \log  \cosh (x\xi_x + \xi_x^2 +y\xi_y-\xi_y^2) dy  ,
\end{align*}
thus,
\begin{align*}
\KL(P_\theta,P_{\theta'}) 
&= \int\limits_\bR \int\limits_\bR \phi(x)\phi(y)  \log \frac{\cosh (x\xi_x+\xi_x^2 +y\xi_y+\xi_y^2)}{ \cosh (x\xi_x + \xi_x^2 +y\xi_y-\xi_y^2)} dx dy \\
&= \int\limits_\bR \phi(z)  \log \frac{\cosh (z\sqrt{\xi_x^2+\xi_y^2}+\xi_x^2 +\xi_y^2)}{ \cosh (z\sqrt{\xi_x^2+\xi_y^2} + \xi_x^2 -\xi_y^2)} dz\\
&= \int\limits_\bR \phi(z)  \log \frac{\cosh (\xi z+\xi_x^2 +\xi_y^2)}{ \cosh (\xi z + \xi_x^2 -\xi_y^2)} dz
\end{align*}
since $\xi_x^2+\xi_y^2=\xi^2$.
By the mean value theorem and the fact that $\tanh$ is monotonically increasing,
\begin{align*}
\log \frac{\cosh (\xi z+\xi_x^2 +\xi_y^2)}{ \cosh (\xi z + \xi_x^2 -\xi_y^2)} 
&\leq 2\xi_y^2\tanh (\xi z+\xi_x^2 +\xi_y^2) \\
&= 2\xi_y^2\tanh (\xi z+\xi^2)
\end{align*}
for all $z$.
Since $\tanh$ is an odd function,
\begin{align*}
\KL(P_\theta,P_{\theta'}) 
&\leq 2\xi_y^2 \int\limits_\bR \phi(z) \tanh(\xi z+\xi^2) dz \\
&=2 \xi_y^2 \int\limits_\bR \phi(z) (\tanh(\xi z+\xi^2) - \tanh(\xi z)) dz.
\end{align*}
Using the mean value theorem again,
\begin{align*}
\tanh(\xi z+\xi^2) - \tanh(\xi z)
&\leq \xi^2 \max_{x\in[\xi z,\xi z+\xi^2]}(1-\tanh^2(x)) \\
&\leq \xi^2  
\end{align*}
for all $z$, so 
\begin{align*}
\KL(P_\theta,P_{\theta'}) 
&\leq 2 \xi^2 \xi_y^2 \\
&= 2 \xi^4 \sin^2\frac{\beta}{2} \\
&= \xi^4 (1-\cos\beta).
\end{align*}
%
\end{proof}}{}

\subsection{Non-sparse lower bound}

\begin{theorem}\label{thm2:lb-nonsparse}
Assume that $d\geq 9$ and $\frac{\lambda}{\sigma}\leq 0.2$.
Then
\begin{align*}
\inf_{\widehat{F}_n}\sup_{\theta\in\Theta_\lambda} \bE_\theta L_\theta(\widehat{F}_n) 
&\geq \frac{1}{500}   \min\left\{\frac{\sqrt{\log 2}}{3}\frac{\sigma^2}{\lambda^2}\sqrt{\frac{d-1}{n}}, \frac{1}{4}\right\}.
\end{align*}
\end{theorem}

\begin{proof}
Let $\xi=\frac{\lambda}{2\sigma}$, and define 
$$
\epsilon=\min\left\{\frac{\sqrt{\log 2}}{3}\frac{\sigma^2}{\lambda}\frac{1}{\sqrt{n}}, \frac{\lambda}{4\sqrt{d-1}}\right\}.
$$
Define $\lambda_0^2=\lambda^2-(d-1)\epsilon^2$.
Let $\Omega=\{0,1\}^{d-1}$.
For $\omega=(\omega(1),...,\omega(d-1))\in\Omega$, let $\mu_\omega=\lambda_0e_d+\sum_{i=1}^{d-1} (2\omega(i)-1)\epsilon e_i $ (where $\{e_i\}_{i=1}^d$ is the standard basis for $\bR^d$).
Let $\theta_\omega=\left(-\frac{\mu_\omega}{2},\frac{\mu_\omega}{2}\right)\in\Theta_\lambda$.

By Proposition \ref{thm2:KL-symmetric},
\begin{align*}
\KL(P_{\theta_\omega},P_{\theta_\nu}) &\leq \xi^4 (1-\cos\beta_{\omega,\nu})
\end{align*}
where 
$$
\cos\beta_{\omega,\nu}=\frac{|\mu_\omega^T\mu_\nu|}{\lambda^2}=1-\frac{2\rho(\omega,\nu)\epsilon^2}{\lambda^2}
$$
and $\rho$ is the Hamming distance, so
\begin{align*}
\KL(P_{\theta_\omega},P_{\theta_\nu}) &\leq \xi^4\frac{2\rho(\omega,\nu)\epsilon^2}{\lambda^2} \\ 
&\leq \xi^4\frac{2(d-1)\epsilon^2}{\lambda^2}. 
\end{align*}

By Proposition \ref{thm2:clusterloss-bnd-new}, since $\cos\beta_{\omega,\nu}\geq \frac{1}{2}$,
\begin{align*}
L_{\theta_\omega}(F_{\theta_\nu})
&\leq \frac{1}{\pi}\tan\beta_{\omega,\nu} \\
&\leq \frac{2}{\pi}\sin\beta_{\omega,\nu} \\
&\leq \frac{4}{\pi} \frac{\sqrt{d-1}\epsilon}{\lambda}
\end{align*}
and
\begin{align*}
L_{\theta_\omega}(F_{\theta_\nu}) &\geq 2g(\xi) \sin\beta_{\omega,\nu} \cos\beta_{\omega,\nu} \\
&\geq g(\xi) \sin\beta_{\omega,\nu} \\
&\geq \sqrt{2}g(\xi) \frac{\sqrt{\rho(\omega,\nu)}\epsilon}{\lambda} 
\end{align*}
where $g(x)=\phi(x)(\phi(x)-x\Phi(-x))$.
By Lemma \ref{thm2:varshamov}, there exist $\omega_0,...,\omega_M\in\Omega$ such that $M\geq 2^{(d-1)/8}$ and
$$
\rho(\omega_i,\omega_j)\geq \frac{d-1}{8}, \;\;\forall\; 0\leq i<j\leq M.
$$
For simplicity of notation, let $\theta_i=\theta_{\omega_i}$ for all $i\in[0..M]$.
Then, for $i\neq j \in[0..M]$,
\begin{align*}
\KL(P_{\theta_i},P_{\theta_j}) &\leq \xi^4\frac{2(d-1)\epsilon^2}{\lambda^2}, 
\end{align*}
and
\begin{align*}
L_{\theta_i}(F_{\theta_j})
&\leq \frac{4}{\pi} \frac{\sqrt{d-1}\epsilon}{\lambda}
\end{align*}
and
\begin{align*}
L_{\theta_i}(F_{\theta_j}) 
&\geq \frac{1}{2}g(\xi) \frac{\sqrt{d-1}\epsilon}{\lambda} .
\end{align*}

Define 
$$
\gamma=\frac{1}{4}(g(\xi)-2\xi^2) \frac{\sqrt{d-1}\epsilon}{\lambda}.
$$
Then for any $i\neq j\in[0..M]$, and any $\widehat{F}$ such that $L_{\theta_i}(\widehat{F})<\gamma$,
\begin{align*}
L_{\theta_i}(F_{\theta_j}) + L_{\theta_i}(\widehat{F}) + \sqrt{\frac{\KL(P_{\theta_i},P_{\theta_j})}{2}}
&<  \left(\frac{4}{\pi}  + \frac{1}{4}(g(\xi)-2\xi^2) + \xi^2 \right) \frac{\sqrt{d-1}\epsilon}{\lambda} \leq \frac{1}{2}
\end{align*}
because, for $\xi\leq 0.1$, by definition of $\epsilon$,
\begin{align*}
\left(\frac{4}{\pi}  + \frac{1}{4}(g(\xi)-2\xi^2) + \xi^2 \right) \frac{\sqrt{d-1}\epsilon}{\lambda} \leq 
2 \frac{\sqrt{d-1}\epsilon}{\lambda} \leq \frac{1}{2}.
\end{align*}
So, by Proposition \ref{thm2:clusterloss-triangle},
\begin{align*}
L_{\theta_j}(\widehat{F})
&\geq L_{\theta_i}(F_{\theta_j})-L_{\theta_i}(\widehat{F}) -\sqrt{\frac{\KL(P_{\theta_i},P_{\theta_j})}{2}} 
\geq \gamma.
\end{align*}
Also,
\begin{align*}
\max_{i\in[M]} \KL(P_{\theta_i},P_{\theta_0})
&\leq (d-1) \xi^4\frac{2\epsilon^2}{\lambda^2} \\
&\leq \frac{ \log M}{9 n}
\end{align*}
because, by definition of $\epsilon$,
\begin{align*}
\xi^4\frac{2\epsilon^2}{\lambda^2} 
&\leq \frac{ \log 2 }{72 n}.
\end{align*}

So by Proposition \ref{thm2:lowerbnd-gen} and the fact that $\xi\leq 0.1$,
\begin{align*}
\inf_{\widehat{F}_n}\max_{i\in[0..M]} \bE_{\theta_i} L_{\theta_i}(\widehat{F}_n) &\geq 0.07\gamma \\
&= 0.07 \frac{1}{4}(g(\xi)-2\xi^2) \frac{\sqrt{d-1}\epsilon}{\lambda} \\
&\geq \frac{1}{500}   \min\left\{\frac{\sqrt{\log 2}}{3}\frac{\sigma^2}{\lambda^2}\sqrt{\frac{d-1}{n}}, \frac{1}{4}\right\}
\end{align*}
and to complete the proof we use the fact that
\begin{align*}
\inf_{\widehat{F}_n}\sup_{\theta\in\Theta_\lambda} \bE_\theta L_\theta(\widehat{F}_n) \geq \inf_{\widehat{F}_n}\max_{i\in[0..M]} \bE_{\theta_i} L_{\theta_i}(\widehat{F}_n).
\end{align*}
\end{proof}

\subsection{Sparse lower bound}

\begin{theorem}\label{thm2:lb-sparse}
Assume that $\frac{\lambda}{\sigma}\leq 0.2$, $d\geq 17$, and
$$5 \leq s\leq \frac{d-1}{4}+1.$$
Then
\begin{align*}
\inf_{\widehat{F}_n}\sup_{\theta\in\Theta_{\lambda,s}} \bE_\theta L_\theta(\widehat{F}_n) 
&\geq  \frac{1}{600} 
\min\left\{ \sqrt{\frac{8}{45}}  \frac{\sigma^2}{\lambda^2} \sqrt{\frac{s-1}{n} \log \left(\frac{d-1}{s-1}\right) }, \frac{1}{2}  \right\} .
\end{align*}
\end{theorem}

\begin{proof}
For simplicity, we state this proof for $\Theta_{\lambda,s+1}$, assuming $4 \leq s\leq \frac{d-1}{4}$.
Let $\xi=\frac{\lambda}{2\sigma}$, and define 
$$
\epsilon =\min\left\{ \sqrt{\frac{8}{45}} \frac{\sigma^2}{\lambda} \sqrt{\frac{1}{n} \log \left(\frac{d-1}{s}\right) }, \frac{1}{2} \frac{\lambda}{\sqrt{s}} \right\}.
$$
Define $\lambda_0^2=\lambda^2-s\epsilon^2$.
Let $\Omega=\{\omega\in\{0,1\}^{d-1}:\; \|\omega\|_0=s\}$.
For $\omega=(\omega(1),...,\omega(d-1))\in\Omega$, let $\mu_\omega=\lambda_0e_d+\sum_{i=1}^{d-1} \omega(i)\epsilon e_i $ (where $\{e_i\}_{i=1}^d$ is the standard basis for $\bR^d$).
Let $\theta_\omega=\left(-\frac{\mu_\omega}{2},\frac{\mu_\omega}{2}\right)\in\Theta_{\lambda,s}$.

By Proposition \ref{thm2:KL-symmetric},
\begin{align*}
\KL(P_{\theta_\omega},P_{\theta_\nu}) &\leq \xi^4 (1-\cos\beta_{\omega,\nu})
\end{align*}
where 
$$
\cos\beta_{\omega,\nu}=\frac{|\mu_\omega^T\mu_\nu|}{\lambda^2}=1-\frac{\rho(\omega,\nu)\epsilon^2}{2\lambda^2}
$$
and $\rho$ is the Hamming distance, so
\begin{align*}
\KL(P_{\theta_\omega},P_{\theta_\nu}) &\leq \xi^4\frac{\rho(\omega,\nu)\epsilon^2}{2\lambda^2} \\ 
&\leq \xi^4\frac{s\epsilon^2}{\lambda^2}. 
\end{align*}

By Proposition \ref{thm2:clusterloss-bnd-new}, since $\cos\beta_{\omega,\nu}\geq \frac{1}{2}$,
\begin{align*}
L_{\theta_\omega}(F_{\theta_\nu})
&\leq \frac{1}{\pi}\tan\beta_{\omega,\nu} \\
&\leq \frac{2}{\pi}\sin\beta_{\omega,\nu} \\
&\leq \frac{2\sqrt{2}}{\pi}\frac{\sqrt{s} \epsilon}{\lambda}
\end{align*}
and
\begin{align*}
L_{\theta_\omega}(F_{\theta_\nu}) &\geq 2g(\xi) \sin\beta_{\omega,\nu} \cos\beta_{\omega,\nu} \\
&\geq g(\xi) \sin\beta_{\omega,\nu} \\
&\geq \frac{g(\xi)}{\sqrt{2}} \frac{\sqrt{\rho(\omega,\nu)}\epsilon}{\lambda} 
\end{align*}
where $g(x)=\phi(x)(\phi(x)-x\Phi(-x))$.
By Lemma \ref{thm2:varshamovsparse}, there exist $\omega_0,...,\omega_M\in\Omega$ such that $\log(M+1)\geq \frac{s}{5}  \log\left(\frac{d-1}{s}\right)$ and
$$
\rho(\omega_i,\omega_j)\geq \frac{s}{2}, \;\;\forall\; 0\leq i<j\leq M.
$$
For simplicity of notation, let $\theta_i=\theta_{\omega_i}$ for all $i\in[0..M]$.
Then, for $i\neq j \in[0..M]$,
\begin{align*}
\KL(P_{\theta_i},P_{\theta_j}) &\leq \xi^4\frac{s\epsilon^2}{\lambda^2},
\end{align*}
and
\begin{align*}
L_{\theta_i}(F_{\theta_j})
&\leq \frac{2\sqrt{2}}{\pi}\frac{\sqrt{s} \epsilon}{\lambda}
\end{align*}
and
\begin{align*}
L_{\theta_i}(F_{\theta_j}) 
&\geq \frac{g(\xi)}{2} \frac{\sqrt{s}\epsilon}{\lambda}  .
\end{align*}
Define 
$$
\gamma=\frac{1}{4}(g(\xi)-\sqrt{2}\xi^2) \frac{\sqrt{s}\epsilon}{\lambda}.
$$
Then for any $i\neq j\in[0..M]$, and any $\widehat{F}$ such that $L_{\theta_i}(\widehat{F})<\gamma$,
\begin{align*}
L_{\theta_i}(F_{\theta_j}) + L_{\theta_i}(\widehat{F}) + \sqrt{\frac{\KL(P_{\theta_i},P_{\theta_j})}{2}}
&< \left( \frac{2\sqrt{2}}{\pi} + \frac{1}{4}(g(\xi)-\sqrt{2}\xi^2) + \frac{\xi^2}{\sqrt{2}}  \right)  \frac{\sqrt{s} \epsilon}{\lambda} \leq \frac{1}{2}
\end{align*}
because, for $\xi\leq 0.1$, by definition of $\epsilon$,
\begin{align*}
\left( \frac{2\sqrt{2}}{\pi} + \frac{1}{4}(g(\xi)-\sqrt{2}\xi^2) + \frac{\xi^2}{\sqrt{2}}  \right)  \frac{\sqrt{s} \epsilon}{\lambda} \leq
\frac{\sqrt{s} \epsilon}{\lambda} \leq \frac{1}{2}.
\end{align*}
So, by Proposition \ref{thm2:clusterloss-triangle},
\begin{align*}
L_{\theta_j}(\widehat{F})
&\geq L_{\theta_i}(F_{\theta_j})-L_{\theta_i}(\widehat{F}) -\sqrt{\frac{\KL(P_{\theta_i},P_{\theta_j})}{2}} 
\geq \gamma.
\end{align*}
Also,
\begin{align*}
\max_{i\in[M]} \KL(P_{\theta_i},P_{\theta_0})
&\leq \xi^4\frac{s\epsilon^2}{\lambda^2} \\
&\leq \frac{1}{18n} \log \left(\frac{d-1}{s}\right)^\frac{s}{5} \\
&\leq \frac{1}{9n} \log \left(\left(\frac{d-1}{s}\right)^\frac{s}{5}-1\right) \\
&\leq \frac{ \log M}{9 n}
\end{align*}
because, by definition of $\epsilon$,
\begin{align*}
\xi^4\frac{s\epsilon^2}{\lambda^2} 
&\leq \frac{s}{90n} \log \left(\frac{d-1}{s}\right) .
\end{align*}
So by Proposition \ref{thm2:lowerbnd-gen} and the fact that $\xi\leq 0.1$,
\begin{align*}
\inf_{\widehat{F}_n}\max_{i\in[0..M]} \bE_{\theta_i} L_{\theta_i}(\widehat{F}_n) &\geq 0.07\gamma \\
&\geq 0.07 \frac{0.1}{4} \frac{\sqrt{s}\epsilon}{\lambda} \\
&\geq  \frac{1}{600} 
\min\left\{ \sqrt{\frac{8}{45}}  \frac{\sigma^2}{\lambda^2} \sqrt{\frac{s}{n} \log \left(\frac{d-1}{s}\right) }, \frac{1}{2}  \right\} 
\end{align*}
and to complete the proof we use the fact that
\begin{align*}
\inf_{\widehat{F}_n}\sup_{\theta\in\Theta_{\lambda,s}} \bE_\theta L_\theta(\widehat{F}_n) \geq \inf_{\widehat{F}_n}\max_{i\in[0..M]} \bE_{\theta_i} L_{\theta_i}(\widehat{F}_n).
\end{align*}
\end{proof}

%
%
%

\end{document}